\documentclass[dvipsnames,10pt]{article} 
\usepackage[preprint]{tmlr}

\usepackage{amsmath,amsfonts,bm}

\def\eqref#1{equation~\ref{#1}}

\def\1{\bm{1}}

\DeclareMathAlphabet{\mathsfit}{\encodingdefault}{\sfdefault}{m}{sl}
\SetMathAlphabet{\mathsfit}{bold}{\encodingdefault}{\sfdefault}{bx}{n}

\usepackage{hyperref}
\usepackage{url}

\usepackage{microtype}
\usepackage{graphicx}
\usepackage{caption}
\usepackage{subcaption}
\usepackage{booktabs} 

\usepackage{algorithm}
\let\classAND\AND
\let\AND\relax
\usepackage[noend]{algorithmic}

\let\AND\classAND
\AtBeginEnvironment{algorithmic}{\let\AND\algoAND}

\usepackage{amsmath}
\usepackage{amssymb}
\usepackage{mathtools}
\usepackage{amsthm}

\usepackage{url}
\usepackage{amsmath}
\usepackage{bbm}
\usepackage{bm}
\usepackage{booktabs}
\usepackage{paralist}
\usepackage{multirow}
\usepackage{makecell}
\usepackage{sidecap}
\usepackage{adjustbox}
\usepackage{hyperref}

\usepackage{xargs}
\usepackage{xspace}

\graphicspath{{figures/}}

\usepackage{etoolbox}
\newcommand{\algorithmicdoinparallel}{\textbf{do in parallel}}
\makeatletter
\AtBeginEnvironment{algorithmic}{%
  \newcommand{\FORALLP}[2][default]{\ALC@it\algorithmicforall\ #2\ %
    \algorithmicdoinparallel\ALC@com{#1}\begin{ALC@for}}%
}
\makeatother

\usepackage[colorinlistoftodos,prependcaption,textsize=tiny]{todonotes}

\newcommandx{\uvc}[2][1=]{\todo[color=Turquoise!50,#1]{\sf \textbf{\"Umit:} #2}\xspace}

\newcommandx{\mfb}[2][1=]{\todo[color=red!50,#1]{\sf \textbf{Fatih:} #2}\xspace}
\newcommandx{\dl}[2][1=]{\todo[color=green!50,#1]{\sf \textbf{DL:}  #2}\xspace}
\newcommandx{\ks}[2][1=]{\todo[color=yellow!50,#1]{\sf \textbf{ks:}  #2}\xspace}

\usepackage[capitalize,noabbrev]{cleveref}

\theoremstyle{plain}
\newtheorem{theorem}{Theorem}[section]

\theoremstyle{definition}

\theoremstyle{remark}

\title{Cooperative Minibatching in Graph Neural Networks}

\author{%
  \name Muhammed Fatih Bal{\i}n\thanks{Part of this work was done during an internship at NVIDIA Corporation in the summer of 2022.} \email balin@gatech.edu \\
  \addr School of Computational Science and Engineering \\ 
  Georgia Institute of Technology, Atlanta, GA, USA
  \AND
  \name Dominique LaSalle \email dlasalle@nvidia.com \\
  \addr NVIDIA Corporation, Santa Clara, CA, USA
  \AND
  \name \"Umit V. \c{C}ataly\"urek\thanks{Amazon Web Services. This publication describes work performed at the Georgia Institute of Technology and is not associated with AWS.} \email umit@gatech.edu \\
  \addr School of Computational Science and Engineering \\ 
  Georgia Institute of Technology, Atlanta, GA, USA
}

\def\month{MM}  
\def\year{YYYY} 
\def\openreview{\url{https://openreview.net/forum?id=XXXX}} 

\begin{document}

\maketitle

\begin{abstract}
Training large scale Graph Neural Networks (GNNs) requires significant computational resources,
and the process is highly data-intensive.
One of the most effective ways to reduce resource requirements is minibatch training 
coupled with graph sampling.
GNNs have the unique property that items in a minibatch have overlapping data. 
However, the commonly implemented Independent Minibatching approach assigns each Processing 
Element (PE, i.e., cores and/or GPUs) its own minibatch to process, leading to duplicated computations and input data access across PEs. 
This amplifies the Neighborhood Explosion Phenomenon (NEP), which is the main bottleneck limiting scaling. 
To reduce the effects of NEP in the multi-PE setting,
we propose a new approach called Cooperative Minibatching. 
Our approach capitalizes on the fact that the size of the sampled subgraph is a concave function of the batch size, leading to 
significant reductions in the amount of work as batch sizes increase. Hence, it is favorable for 
processors equipped with a fast interconnect to work on a large minibatch together as a single larger processor, instead of working on separate smaller 
minibatches, even though global batch size is identical.
We also show how to take advantage of the same phenomenon in serial execution by generating dependent consecutive minibatches. 
Our experimental evaluations show up to 4x bandwidth savings for fetching vertex embeddings, by simply increasing 
this dependency without harming model convergence. Combining our proposed approaches, we achieve up to 64\% 
speedup over Independent Minibatching on single-node multi-GPU systems, using same resources.
\end{abstract}

\section{Introduction}

Graph Neural Networks (GNNs) have become de facto deep learning models for
unstructured data, achieving state-of-the-art results
on various application domains involving graph data such as recommendation
systems~\citep{recommendersystemssurvey,pinsage}, fraud detection~\citep{fraudgnn2,fraudgnn}, identity
resolution~\citep{Xu_2019},
and traffic forecasting~\citep{trafficforecastingsurvey}.
However, as the usage of technology continues to increase,
the amount of data generated by these applications is growing exponentially,
resulting in large and complex graphs
that are infeasible or too time-consuming to train on a single processing element~\citep{pinsage, aligraph}.
For example, some social media graphs are reaching billions of vertices and trillions of interactions~\citep{trillionfacebook}.
Efficient distributed training of GNNs is essential for extracting value from large-scale unstructured data that exceeds the cost of storing and processing such data.

Due to the popularity of Deep Neural Networks (DNNs) and the need to support larger models and datasets,
a great deal of research has focused on increasing the scale and efficiency of distributed DNN training.
Techniques such as data parallelism~\citep{ginsburg2018large, imagenetonehour},
pipelining~\citep{pipelining}, and intra-layer parallelism~\citep{intralayer}
have been employed.
Following the success of traditional distributed DNN training, the same techniques
have also been adapted to GNN training, such as data
parallelism~\citep{gandhi2021p3, pagraph, zheng2021distributed, aligraph} and
intra-layer parallelism~\citep{CAGNET}.

The parallelization techniques mentioned earlier are used to scale both full-batch and minibatch training in a distributed setting.
Minibatch training~\citep{dnnminibatch} is the
go-to method to train DNN models as it outperforms
full-batch training in terms of convergence~\citep{sgd_vs_gd, efficient_minibatch_training, keskar2016large, wilson2003general}, and more recently has been shown to also offer
the same benefit for GNNs~\citep{zheng2021distributed}.
In the distributed setting, minibatch training for DNNs using data parallelism is straightforward.
The training samples are partitioned across the Processing Elements (PE) and they compute
the forward/backward operations on their minibatches.
The only communication required is an all-reduce operation for the gradients.

Unfortunately, minibatch training a GNN model is more challenging than a usual DNN model.
GNNs turn a given graph encoding relationships into computational dependencies.
Thus in an $L$-layer GNN model, each minibatch computation has a different
structure as it is performed on the $L$-hop neighborhood of the minibatch vertices.
Real-world graphs usually are power law graphs~\citep{powerlawgraphs} with small diameters, thus, it is a challenge
to train deep GNN models as the $L$-hop neighborhood grows exponentially w.r.t. $L$,
reaching almost the whole graph within a few hops.
This is especially the case for the node classification and link prediction scenarios where the graph can be large and a single connected component contains almost all the data.

Very large GNN datasets necessitate storing the graph and node embeddings on slower storage mediums. To enable 
GNN training efficiently in such a setting, several techniques have been
proposed~\citep{ginex2022, waleffe2022mariusgnn}. These studies assume that the graph and its features are stored on 
disks or SSDs and design their systems to reduce data transfers. The methods proposed in this paper directly apply to these settings by reducing bandwidth requirements, as seen in~\cref{secc:experiments}.

A single epoch of full-batch GNN training requires computation proportional to the number
of layers $L$ and the size of the graph. However, minibatch
training requires more operations to process a single epoch due to repeating calculations in the
$2$nd through $L$-th layers.
As the batch size decreases, the number of repeated calculations increases. This is because the
vertices and edges have to be processed each time they appear in the $L$-hop neighborhood.
Thus, it is natural to conclude that using effectively larger batch sizes in GNNs reduces the
number of computations and data accesses of an epoch in contrast to
regular DNN models. Our contributions in this work
utilizing this important observation can be listed as follows:
\begin{compactitem}
    \item Investigate work vs. batch size relationship and present theorems stating
    the cost of processing a minibatch is a concave function of the batch size
    (\cref{th:work_monotonicity,th:overlap_monotonicity}).
    \item Utilize this relationship by combining data and intra-layer parallelism
    to process a minibatch across multiple PEs for reduced work (\cref{subsecc:cooperating_pes}),
    with identical global batch size. We call this new approach {\em Cooperative Minibatching}.
    \item Use the same idea to generate consecutive dependent minibatches to increase temporal
    vertex embedding access locality (\cref{subsecc:dependent_batches}).
    {\em Dependent Minibatching} can reduce the transfer amount of vertex embeddings up to
    $4\times$, without harming model convergence.
    \item Show that the two approaches are orthogonal.
    When combined, they result in up to $1.64\times$ speedup over Independent Minibatching with identical global batch size.
    \item Further show the applicability of Cooperative Minibatching to the subgraph
    sampling methods and prove that sampled subgraph density is a nondecreasing
    function of the batch size (\cref{th:subgraph_density_monotonicity}).
\end{compactitem}


\section{Background}
\label{secc:background}

A graph $\mathcal{G} = (V, E)$ consists of vertices $V$ and edges $E \subset V \times V$ along
with optional edge weights $A_{ts} > 0, \forall (t \to s) \in E$. Given a vertex
$s$, the $1$-hop neighborhood $N(s)$ is defined as $N(s) = \{t | (t \to s) \in E
\}$, and it can be naturally expanded to a set of vertices
$S$ as $N(S) = \cup_{s \in S} N(s)$.

GNN models work by passing previous layer embeddings ($H$) from $N(s)$
to $s$, and then combining them using a nonlinear function $f^{(l)}$ at layer
$l$, given initial vertex features $H^{(0)}$:

\begin{equation}
    \label{eq:gnn}
    H_s^{(l + 1)} = f^{(l)}(H_s^{(l)}, \{ H_t^{(l)} \mid t \in N(s)\})
\end{equation}

If the GNN model has $L$ layers, then the loss is computed by taking the final layer $L$'s embeddings and averaging their losses
over the set of training vertices $V_t \subset V$ for {\em full-batch} training.
In $L$-layer full-batch training, the total number of vertices that
needs to be processed is $L |V|$.

\subsection{Minibatching in GNNs}

In minibatch training, a random subset of training vertices, called {\em Seed
Vertices}, is selected, and training is done over the (sampled) subgraph
composed of $L$-hop neighborhood of the seed vertices.
On each iteration, minibatch training computes the loss on seed vertices, which are random subsets
of the training set $V_t$.

Given a set of vertices $S$, we define $l$-th layer expansion set, or the $l$-hop neighborhood $S^l$ as:
\begin{equation}
    \label{eq:s_l}
    S^0 = S,\ \ \  S^{(l + 1)} = S^l \cup N(S^l)
\end{equation}
For GNN computations, $S^l$ would also denote the set of
the required vertices to compute~(\ref{eq:gnn}) at each layer $l$.
Using the same notation, $\{ s \}^l$ denotes $l$-layer expansion set
starting from single vertex $s\in V$.

For a single minibatch iteration, the
total number of vertices that need to be processed is $\sum_{l=1}^L |S^l|$.
There are $\frac{|V|}{|S^0|}$ minibatches assuming $V_t = V$.
Since each $|S^l| \geq |S^0|$, and a single epoch of minibatch training needs to
go over the whole dataset, the work $W(|S^0|)$ for a single epoch is:
\begin{equation}
    \label{eqc:work_definition}
    W(|S^0|) = \frac{|V|}{|S^0|}\sum_{l=1}^L E[|S^l|] \geq \frac{|V|}{|S^0|}\sum_{l=1}^L |S^0| = L |V|
\end{equation}
where $E[|S^l|]$ is the expected number of sampled vertices in layer $l$ and
$|S^0|$ is the batch size.
That is, the total amount of work to process a single epoch increases over full-batch training.
The increase in work due to minibatch training is thus
encoded in the ratios $\frac{E[|S^l|]}{|S^0|}, 1 \leq l \leq L$.

When sampling is used with minibatching, the minibatch subgraph may potentially
become random.
However, the same argument for the increasing total amount of work
holds for them too, as seen in~\cref{figc:num_input_nodes}.

\subsection{Graph Sampling}
\label{subsecc:graph_sampling}


We focus on samplers whose expected number of sampled vertices is a
function of the batch size. Neighbor Sampling (NS)~\citep{neighborsampling}, LABOR~\citep{Balin23-NeurIPS} and RandomWalk (RW) Sampling~\citep{pinsage} all have this property and they are all applied recursively for GNN models with multiple layers. \cref{subsecc:graph_sampling_background} describes the details of these sampling methods.

\begin{SCfigure}
    \centering
    \includegraphics[width=0.45\linewidth]{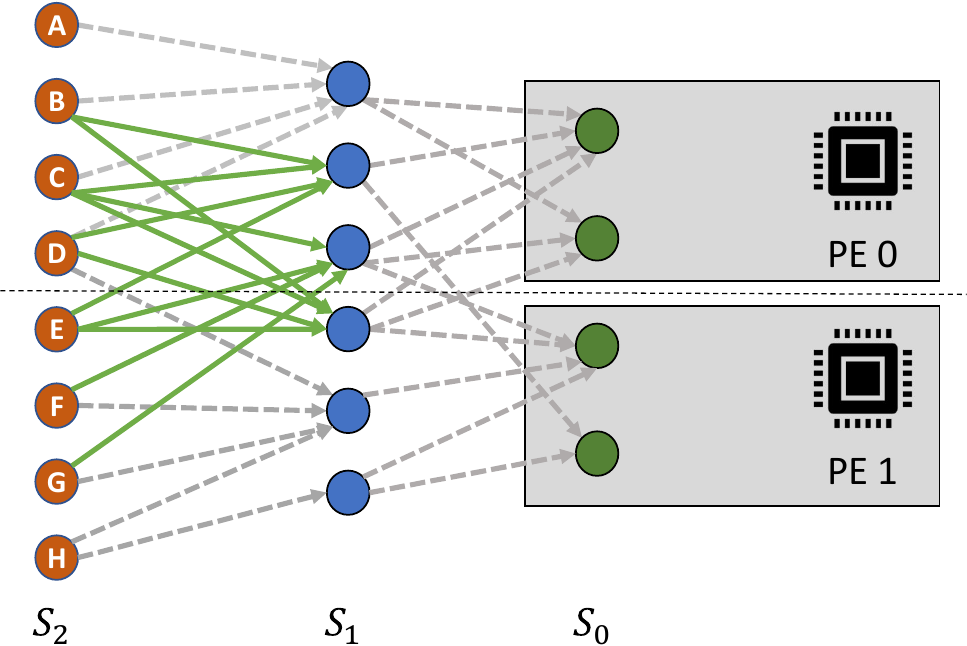}
    \caption{\small Minibatches of two processing elements may share edges in the
    second layer and vertices starting in the first layer. For independent minibatching,
    the solid green edges shared by both processing elements represent
    duplicate work, and input nodes B through G are duplicated along with the
    directed endpoints of green edges. However for cooperative
    minibatching, the vertices and edges are partitioned across the PEs with no duplication,
    and the edges crossing the line between the two PEs necessitate communication.}
    \label{figc:duplicate_work}
    \vspace*{-3ex}
\end{SCfigure}

\subsection{Independent Minibatching}

Independent minibatching is commonly used in multi-GPU, and distributed GNN
training frameworks~\citep{dsp2023, gandhi2021p3, pagraph, zheng2021distributed, aligraph} to
parallelize the training and allows scaling to larger problems.
Each Processing Element (PE, e.g., GPUs, CPUs, or cores of multi-core CPU),
starts with their own $S^0$ of size $b$ as the seed vertices, and
compute $S^1, \dots, S^L$ along with the sampled edges to generate minibatches (see~\cref{figc:duplicate_work}).
Computing $S^1, \dots, S^L$ depends on the chosen sampling algorithm, such as the ones explained in~\cref{subsecc:graph_sampling_background}.
It has the advantage that doing a forward/backward pass does
not involve any communication with other PEs after the initial minibatch
preparation stage at the expense of duplicate work. 

\section{Cooperative Minibatching}
\label{secc:cooperative_minibatching}

In this section, we present two theorems that show the work of an epoch
will be monotonically nonincreasing with increasing batch sizes.
After that, we propose two algorithms that can take advantage of this monotonicity.
Due to lack of space, we provide their full proofs in~\cref{subsecc:work_monotonicity,subsecc:overlap_monotonicity}.

\begin{theorem}
The work per epoch $\frac{E[|S^l|]}{|S^0|}$ required to train a GNN model using minibatch training is
monotonically nonincreasing as the batch size $|S^0|$ increases.
\label{th:work_monotonicity}
\end{theorem}

\begin{proof} (Idea)
Given any random $S^0$ and $l \geq 1$, we have the corresponding set $S^l$.
We can shrink $S^0$ by removing random elements $s \in S^0$ from it to get
$S^0 - \{s\}$. If we compare $(S^0 - \{s\})^l$ to $S^l$, we will see that any 
difference will come from elements that are contained only in $\{s\}^l$. If we count
such elements for each $s \in S^0$ and sum them, it will be bounded by $|S^l|$.
More concretely if we let $S'^0 = S^0 - \{s\}$ be a random subset of $S^0$ with one 
element missing, we have:
\begin{equation}
    \sum_{s \in S^0} |S^l| - |(S^0 - \{s\})^l| \leq |S^l| \iff |S^0||S^l| - |S^l| \leq \sum_{s \in S^0} |(S^0 - \{s\})^l| \iff |S^l||S'^0| \leq |S^0|E[|S'^l|]
\end{equation}
\end{proof}

Empirical evidence can be seen in~\cref{figc:num_input_nodes,figc:num_input_nodes_2}. In
fact, the decrease is related to the cardinality of the following set:
\begin{equation}
    \label{eq:T_1_S}
    T^l(S) = \{w \in S^l \mid w \in \{s\}^l, \exists! s \in S^0\}
\end{equation}

When $T(S^0)$ is equal to $S^l$, the work is equal as well. In the next section,
we further investigate the effect of $|T(S^0)|$ on $E[|S^l|]$.

In addition to the definition of $T(S)$ above, if we define the following set $T_2(S)$:
\begin{equation}
    \label{eq:T_2_S}
    T_2^l(S) = \{w \in S^l \mid w \in \{s\}^l \cap \{s'\}^l, \exists! \{s, s'\} \subset S^0\}
\end{equation}

\begin{theorem}
The expected subgraph size $E[|S^l|]$ required to train a GNN model using minibatch training is a
concave function of batch size, $|S^0|$.
\label{th:overlap_monotonicity}
\end{theorem}

\begin{proof} (Idea)
Note that we have:
\begin{gather*}
    |T^l(S^0)| - 2|T_2^l(S^0)| = \sum_{\substack{S'^0 \subset S^0 \\ |S'^0| + 1 = |S^0|}} |T^l(S^0)| - |T^l(S'^0)| = |S^0| |T^l(S^0)| - |S^0| E[|T^l(S'^0)|] \\
    \implies (|S^0| - 1) |T^l(S^0)| \leq |S^0| E[|T^l(S'^0)|] \implies \frac{|T^l(S^0)|}{|S^0|} \leq \frac{E[|T^l(S'^0)|]}{|S'^0|}
\end{gather*}
We have just shown that $\frac{E[|T^l(S^0_i)|]}{i}$ is monotonically nonincreasing as $i$ increases. If we let $S^0_i$ denote a random subset of $S^0$ of size $i$, we have:
\begin{equation}
    E[|S^l|] = \sum_{i=1}^{|S^0|} \frac{E[|T^l(S^0_i)|]}{i}
\end{equation}
As $E[|S^l|]$ is the sum of monotonically nonincreasing parts, we conclude that
it is a concave function of the batch size, $|S^0|$.
\end{proof}

So, the slope of the expected number of sampled vertices flattens as batch size
increases, see the last row in~\cref{figc:num_input_nodes} and the first row in~\cref{figc:num_input_nodes_2}. Note that
this implies work monotonicity as well.

\subsection{Cooperative Minibatching}
\label{subsecc:cooperating_pes}

As explained in~\cref{secc:background}, Independent Minibatching can not take advantage of
the reduction in work with increasing global batch sizes and number of PEs, because
it uses separate small, {\em local}, batches of sizes $b$ on each PE for each step of training.
On the other hand, one can also keep the {\em global} batch size constant, $b P = |S^0|$,
and vary the number of processors $P$. As $P$ increases, Independent Minibatching will perform more and more
duplicate work because the local batch size is a decreasing function, $b=\frac{|S^0|}{P}$, of $P$.

\begin{algorithm}
   \caption{Cooperative minibatching}
   \label{alg:cooperative_minibatching}
\begin{algorithmic}
   \STATE {\bfseries Input:} seed vertices $S^0_p$ for each PE $p \in P$, \# layers $L$
   \FORALL[Sampling]{$l \in \{0, \dots, L - 1\}$}
     \FORALLP{$p \in P$}
        \STATE Sample next layer vertices $\tilde{S}^{l+1}_p$ and edges $E^{l}_p$ for $S^{l}_p$
        \STATE all-to-all to redistribute vertex ids for $\tilde{S}^{l+1}_p$ to get $S^{l+1}_p$
     \ENDFOR
   \ENDFOR{}
   \FORALLP[Feature Loading]{$p \in P$}
    \STATE Load input features $H^L_p$ from Storage for vertices $S^L_p$
    \STATE all-to-all to redistribute $H^L_p$ to get $\tilde{H}^L_p$
   \ENDFOR
   \FORALL[Forward Pass]{$l \in \{L - 1, \dots, 0\}$}
     \FORALLP{$p \in P$}
        \IF{$l + 1 < L$}
            \STATE all-to-all to redistribute $H^{l+1}_p$ to get $\tilde{H}^{l+1}_p$
        \ENDIF
        \STATE Forward pass on bipartite graph $\tilde{S}^{l+1}_p \to S^l_p$ with edges $E^l_p$, input $\tilde{H}^{l+1}_p$ and output $H^l_p$
     \ENDFOR
   \ENDFOR{}
   \FORALLP{$p \in P$}
    \STATE Compute the loss and initialize gradients $G^0_p$
   \ENDFOR
   \FORALL[Backward Pass]{$l \in \{0, \dots, L - 1\}$}
     \FORALLP{$p \in P$}
        \STATE Backward pass on bipartite graph $S^l_p \to \tilde{S}^{(l+1)}_p$ with edges $E^l_p$, input $G^l_p$ and output $\tilde{G}^{l+1}_p$
        \IF{$l + 1 < L$}
            \STATE all-to-all to redistribute $\tilde{G}^{l+1}_p$ to get $G^{l+1}_p$
        \ENDIF
     \ENDFOR
   \ENDFOR{}
\end{algorithmic}
\end{algorithm}

Here, we propose the {\em Cooperative Minibatching} method that will take
advantage of the work reduction with increasing batch sizes in multi-PE settings.
In Cooperative Minibatching, a single global batch of size $b P$ will be processed by all
the $P$ PEs in parallel, eliminating any redundant work across PEs.

We achieve this as follows: we first partition the graph in 1D fashion by logically assigning
each vertex and
its incoming edges to PEs as $V_p$ and $E_p$ for each PE $p$. Next, PE $p$ samples its 
batch seed vertices $S^l_p$ from the training
vertices in $V_p$ for $l=0$ of size $b$. Then using any sampling algorithm,
PE $p$ samples the incoming edges $E^l_p$ from $E_p$ for its seed vertices. Each PE then computes
the set of vertices sampled $\tilde{S}^{l+1}_p = \{ t \mid (t \to s) \in E^l_p \}$.
Note that, $\tilde{S}^{l+1}_p$ has elements residing on different PEs. The PEs exchange the vertex ids 
$\tilde{S}^{l+1}_p$ so that each PE receives the set $S^{l+1}_p \in V_p$. This process is repeated recursively for GNN models 
with multiple layers by using $S^{l+1}_p$ as the seed vertices for the next layer. The
exchanged information is cached to be used during the forward/backward passes.

For the forward/backward passes, the same communication pattern used during cooperative sampling is
used to send and receive
input and intermediate layer embeddings before each GNN layer invocation.
\cref{alg:cooperative_minibatching} details cooperative sampling
and cooperative forward/backward passes for a single GNN training iteration. Independent 
minibatching works the same except that it lacks the all-to-all operations and has 
$\tilde{A}^{l+1}_p = A^{l+1}_p$ for any given variable $A$ instead.
The redistribution of vertices during sampling happens according to the
initial graph partitioning and the rest of the redistribution operations follow
the same communication pattern, always converting a variable $\tilde{A}_p^{l+1}$ into
$A_p^{l+1}$ during the forward pass and $A_p^{l+1}$ into
$\tilde{A}_p^{l+1}$ during sampling and the backward passes for any variable $A$.
Cooperative Minibatching was first proposed in~\cite{cooperative_minibatching_aug_2022, cooperative_minibatching_jan_2023} 
and was later re-discovered in~\cite{polisetty2023gsplit} as Cooperative Training, which was used
during the feature loading and the forward/backward passes of their work while sampling of each
minibatch was performed on a single CPU thread and the edges were split among the GPUs.
We refer the reader to~\cref{subsecc:redundancy_free_connection}
to see the relation between the approach proposed here and the work of~\cite{jia2020} on
redundancy-free GCN aggregation.

\begin{table*}[ht!]
    \caption{Algorithmic complexities of different stages of GNN training at layer $0 \leq l < L$ with $L$ total layers and batch size $B = |S^0|$ with $P$ PEs. Note that $|S^l_p(B)| = |S^l(B)|\frac{1}{P}$, $|E^l_p(B)| = |E^l(B)|\frac{1}{P}$ since the PEs process the partitioned subgraphs. Feature loading happens only at layer $L$.}
    \label{tabc:complexity}
    \begin{adjustbox}{width=\linewidth,center}
    \begin{tabular}{c | r | r}
    \toprule
    \textbf{Stage} & \textbf{Independent} & \textbf{Cooperative} \\
    \midrule
    Sampling & $\mathcal{O}(|S^l(\frac{B}{P})|\frac{1}{\beta})$ & $\mathcal{O}(|S^l_p(B)|\frac{1}{\beta} + |\tilde{S}^{l+1}_p(B)|\frac{c}{\alpha})$ \\
    Feature loading & $\mathcal{O}(|S^L(\frac{B}{P})|\frac{d\rho}{\beta})$ & $\mathcal{O}(|S^L_p(B)|\frac{d\rho}{\beta} + |\tilde{S}^L_p(B)|\frac{dc}{\alpha})$ \\
    Forward/Backward & $\mathcal{O}(M(S^l(\frac{B}{P}), E^l(\frac{B}{P}), S^{l+1}(\frac{B}{P}))\frac{d}{\gamma})$ & $\mathcal{O}(M(S^l_p(B), E^l_p(B), \tilde{S}^{l+1}_p(B))\frac{d}{\gamma} + |\tilde{S}^{l+1}_p(B)|\frac{dc}{\alpha})$ \\
    \bottomrule
    \end{tabular}
    \end{adjustbox}
\end{table*}

\textbf{Complexity Analysis:} Let $M(V_1, E, V_2)$ denote the work to process a bipartite graph $V_2 \to V_1$ with edges 
$E$ for a given GNN model $M$. Assuming cross PE communication bandwidth $\alpha$, Storage (e.g., disk, network, or DRAM) to 
PE bandwidth as $\beta$ and PE memory bandwidth $\gamma$, and cache miss rate $\rho$, we have 
the time complexities given in~\cref{tabc:complexity} to perform different stages of GNN training per PE.
We also use $d$ for embedding 
dimension and $c < 1$ for the cross edge ratio, note that $c \approx \frac{P - 1}{P}$ for random 
partitioning, and smaller for smarter graph partitioning with $P$ standing for the number of PEs. Also the sizes of $\tilde{S}^l$ 
become smaller when graph partitioning is used due to increased overlap.

We see that $\gamma \approx 2\text{ TB/s}$ , $\alpha \approx 300\text{ GB/s}$ 
and $\beta \approx 30\text{ GB/s}$ in today's modern multi-GPU systems~\citep{dgxa100}. Due to $\alpha$ being relatively fast compared to
$\frac{\gamma}{M}$ and $\beta$, our approach brings performance benefits. On newer systems, the
all-to-all bandwidth continues to increase~\citep{dgxh100}, decreasing the cost of cooperation on a global mini-batch. However, on systems where the interconnect does not provide full bandwidth for all-to-all operations, our approach is limited in the speedup it can provide.
Our approach is most applicable for
systems with relatively fast all-to-all bandwidth $\frac{\alpha}{c}$ compared to $\frac{\gamma}{M}$ and $\beta$ and 
large $P$. In particular, starting from $P=2$, cooperative starts to outperform independent even on F/B with the 
mag240M dataset and the R-GCN model in~\cref{subsecc:cooperative_minibatching_experiment,tabc:cooperative_runtimes,tabc:cooperative_speedup}.
More discussion on this topic can be found in~\cref{subsecc:complexity_analysis}.

\subsection{Cooperative Dependent Minibatching}
\label{subsecc:dependent_batches}

Just as any parallel algorithm can be executed sequentially, we can reduce the number of distinct data accesses by having a
single PE process $b$-sized parts of a single $\kappa b$-sized minibatch for $\kappa$ iterations.
In light of~\cref{th:work_monotonicity,th:overlap_monotonicity}, consider doing the following: choose $\kappa \in \mathbb{Z}^+$,
then sample a batch $S^0$ of size $\kappa b$, i.e., $\kappa b=|S^0|$ to
get $S^0, \dots, S^L$. Then sample $\kappa$ minibatches $S^0_i$, of size $b=|S^0_i|$ from
this batch of size $\kappa b$ to get $S^0_i, \dots, S^L_i, \forall i \in \{0,
\dots, \kappa - 1\}$. In the end, all of the input
features required for these minibatches will be a subset of the input features of the large batch, i.e. $S^j_i \subset S^j, \forall i, j$. This means that the
collective input feature requirement of these $\kappa$ batches will be $|S^L|$,
the same as our batch of size $\kappa b$. Hence, we can now take advantage
of the concave growth of the work in~\cref{th:overlap_monotonicity,figc:num_input_nodes}.


Note that, if one does not use any sampling algorithms and proceeds to use the
full neighborhoods, this technique will not give any benefits, as by definition,
the $l$-hop neighborhood of a batch of size $\kappa b$ will always be equal to the
union of the $l$-hop neighborhoods of batches of sizes $b$. However for
sampling algorithms, any overlapping vertex sampled by any two batches of
sizes $b$ might end up with different random neighborhoods resulting in a
larger number of sampled vertices. Thus, having a single large batch
ensures that only a single random set of neighbors is used for any vertex
processed over a period of $\kappa$ batches.

The approach described above has a nested iteration structure and the
minibatches part of one $\kappa$ group will be significantly different than
another group and this might slightly affect convergence. In~\cref{subsecc:smoothed_dependent_minibatching}, we propose
an alternative smoothing approach that does not require two-level nesting and still takes advantage of
the same phenomenon for the NS and LABOR sampling algorithms.

\begin{figure}
    \centering
    \begin{subfigure}{0.30\linewidth}
    \includegraphics[width=\linewidth,trim=3cm 1cm 3cm 1cm,clip]{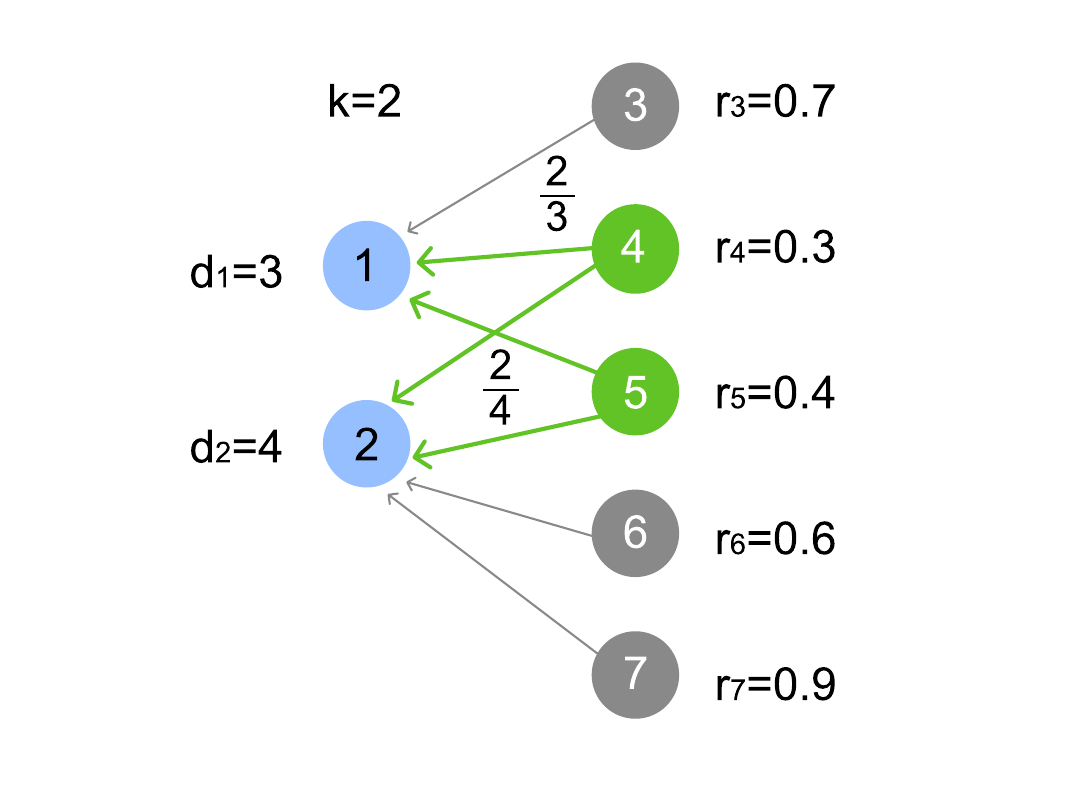}
    \end{subfigure}
    \hspace*{2ex}
    \begin{subfigure}{0.30\linewidth}
    \includegraphics[width=\linewidth,trim=3cm 1cm 3cm 1cm,clip]{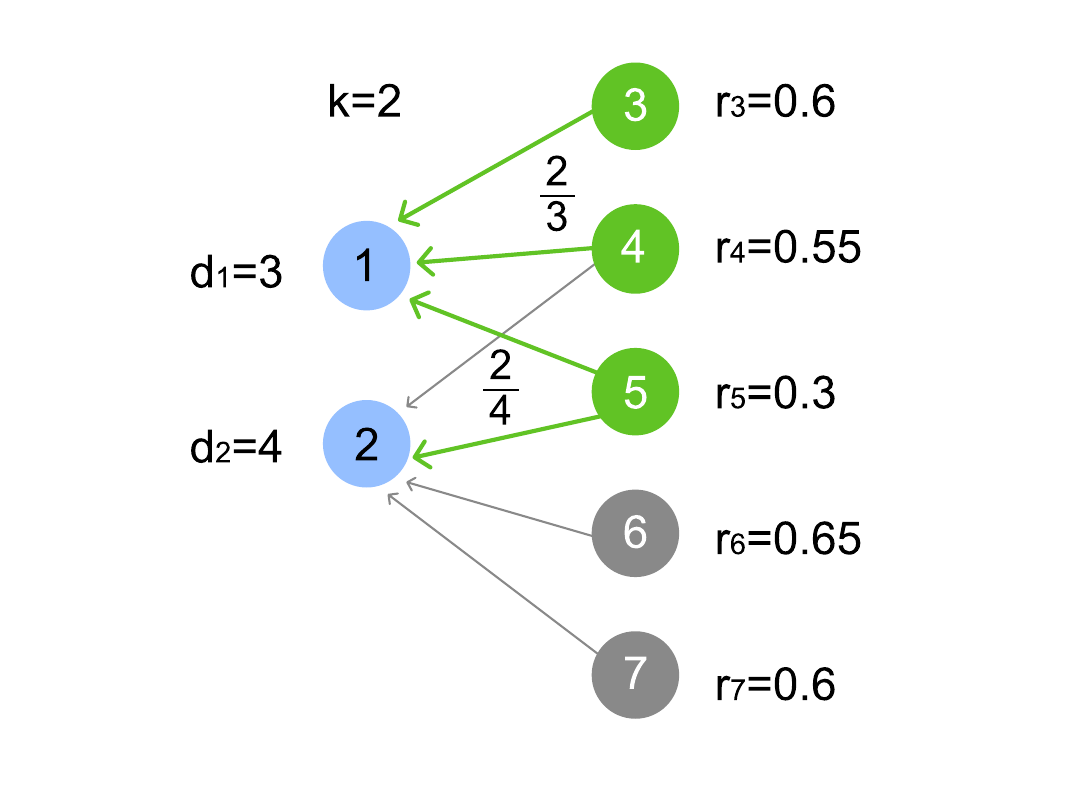}
    \end{subfigure}
    \hspace*{2ex}
    \begin{subfigure}{0.30\linewidth}
    \includegraphics[width=\linewidth,trim=3cm 1cm 3cm 1cm,clip]{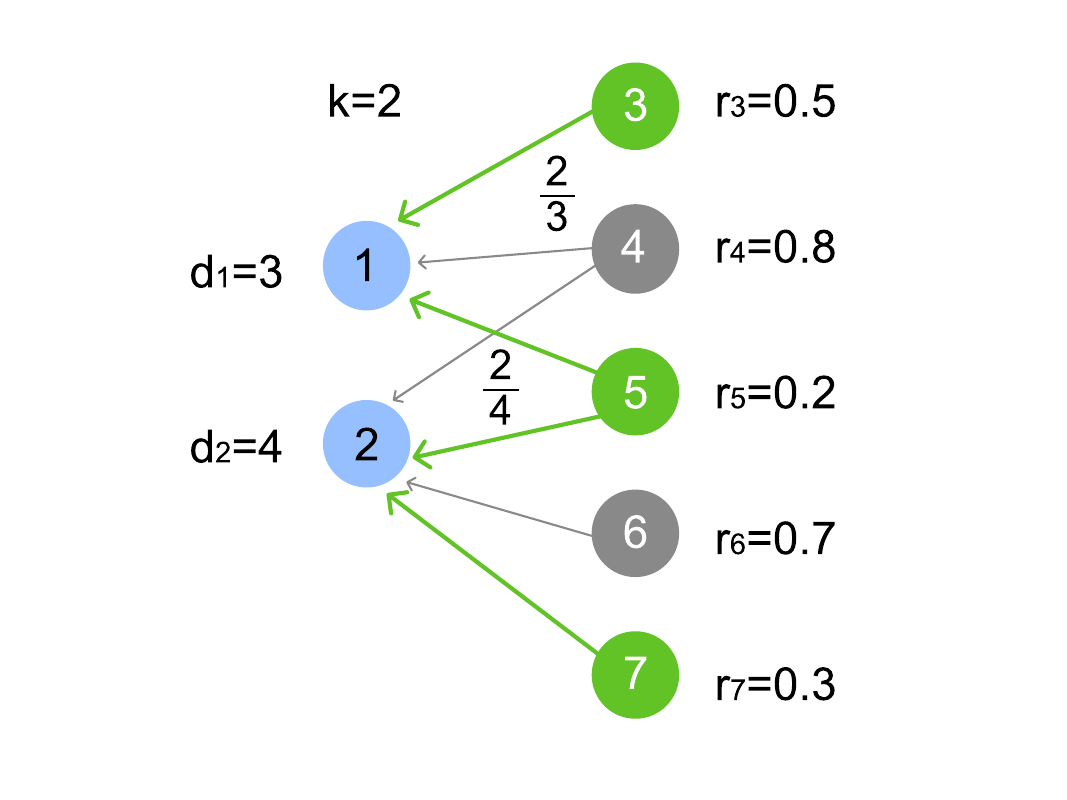}
    \end{subfigure}
    \vspace*{-1ex}
    \caption{A smoothed dependent minibatching example for $\kappa=2$. The middle minibatch is interpolated between the two independent minibatches on the left and the right by interpolating the random numbers used during sampling.}
    \label{figc:dependent_minibatching_example}
\end{figure}

The main idea of our smoothing approach is as follows: each time one samples the neighborhood of a vertex,
normally it is done independently of any previous 
sampling attempts. If one were to do it fully dependently, then one would end up with an identical sampled 
neighborhood at each sampling attempt. What we propose is to do something inbetween, so that the sampled
neighborhood of a vertex changes slowly over time. The speed of change in the sampled
neighborhoods is $\frac{1}{\kappa}$, and after every $\kappa$ iterations, one gets fully independent new 
random neighborhoods for all vertices, see~\cref{figc:dependent_minibatching_example}.
We will experimentally evaluate the locality benefits and the overall effect of this
algorithm on convergence in~\cref{subsecc:dependent_batches_experiment,subsecc:cooperative_dependent_batches}, and
more details on our smoothing approach are discussed in~\cref{subsecc:smoothed_dependent_minibatching}.

\subsection{Relation to subgraph sampling methods}
\label{subsecc:relation_to_subgraph_sampling}

Existing subgraph sampling methods~\cite{Chiang2019, graphsaint-iclr20, Hu2020a, shaDow, pmlr-v139-fey21a, shi2023lmc} randomly sample a subset $S$ of the vertices and
use the same subset for all layers of GNN training, $S = S^0 = \dots = S^L$. The
edges $S_E$ used for training are obtained by taking the vertex induced subgraph
$S_E = \{t \to s \mid (t \to s) \in E \land t, s \in S\}$. We utilize 
the same observation as \cref{th:work_monotonicity,th:overlap_monotonicity} and state:
\begin{theorem}
The expected sampled subgraph density $\frac{E[|S_E|]}{|S|}$ is nondecreasing as the batch size $|S|$ increases.
\label{th:subgraph_density_monotonicity}
\end{theorem}
Thus, having multiple PEs cooperate and work on a larger minibatch together instead
of processing independent smaller minibatches is theoretically expected to converge
faster, as the higher density leads to better approximations. The proof
of~\cref{th:subgraph_density_monotonicity} is provided
in~\cref{subsecc:subgraph_density_monotonicity}.

\section{Experimental Evaluation}
\label{secc:experiments}

We first compare how the work to process an epoch changes w.r.t. to the batch size to empirically
validate~\cref{th:work_monotonicity,th:overlap_monotonicity} for different graph sampling algorithms.
Next, we show how dependent batches introduced in~\cref{subsecc:dependent_batches}
benefits GNN training. We also show the runtime benefits of cooperative
minibatching compared to independent minibatching in the multi-GPU setting. Finally, we show
that these two techniques are orthogonal, can be combined to get multiplicative savings.
~\cref{tabc:dataset} lists details of the datasets we used in experiments. Details
on our experimental setup can be found in~\cref{subsecc:experimental_setup}.

\begin{table*}[h!]
\begin{center}
    \caption{Traits of datasets used in experiments: numbers of vertices,
      edges, avg. degree, features, cached vertex embeddings, and training,
      validation and test vertex splits. Last column has \# minibatches
      in an epoch during model training with 1024 batch size including validation.}
    \label{tabc:dataset}
   \begin{adjustbox}{width=\linewidth,center}
    \begin{tabular}{c | r | r | r | r | r | r | r}
    \toprule
    \textbf{Dataset} & \textbf{$\bm{|V|}$} & \textbf{$\bm{|E|}$} & \textbf{$\bm{\frac{|E|}{|V|}}$} & \textbf{\# feats.} & \textbf{cache size} & \textbf{train - val - test (\%)} & \textbf{\# minibatches} \\
    \midrule
    flickr & 89.2K & 900K & 10.09 & 500 & 70k & 50.00 - 25.00 - 25.00 & 65 \\
    yelp & 717K & 14.0M & 19.52 & 300 & 200k & 75.00 - 10.00 - 15.00 & 595 \\
    reddit & 233K & 115M & 493.56 & 602 & 60k & 66.00 - 10.00 - 24.00 & 172 \\
    papers100M & 111M & 3.2B & 29.10 & 128 & 2M & \ \  1.09 - \ \  0.11 - \ \  0.19 & 1300 \\
    mag240M & 244M & 3.44B & 14.16 & 768 & 2M & \ \  0.45 - \ \  0.06 - \ \  0.04 & 1215 \\
    \bottomrule
    \end{tabular}
    \end{adjustbox}
\end{center}
    \vspace*{-2ex}
\end{table*}

\subsection{Demonstrating monotonicity of work}

\begin{figure*}[ht!]
    \centering
    \includegraphics[width=\linewidth]{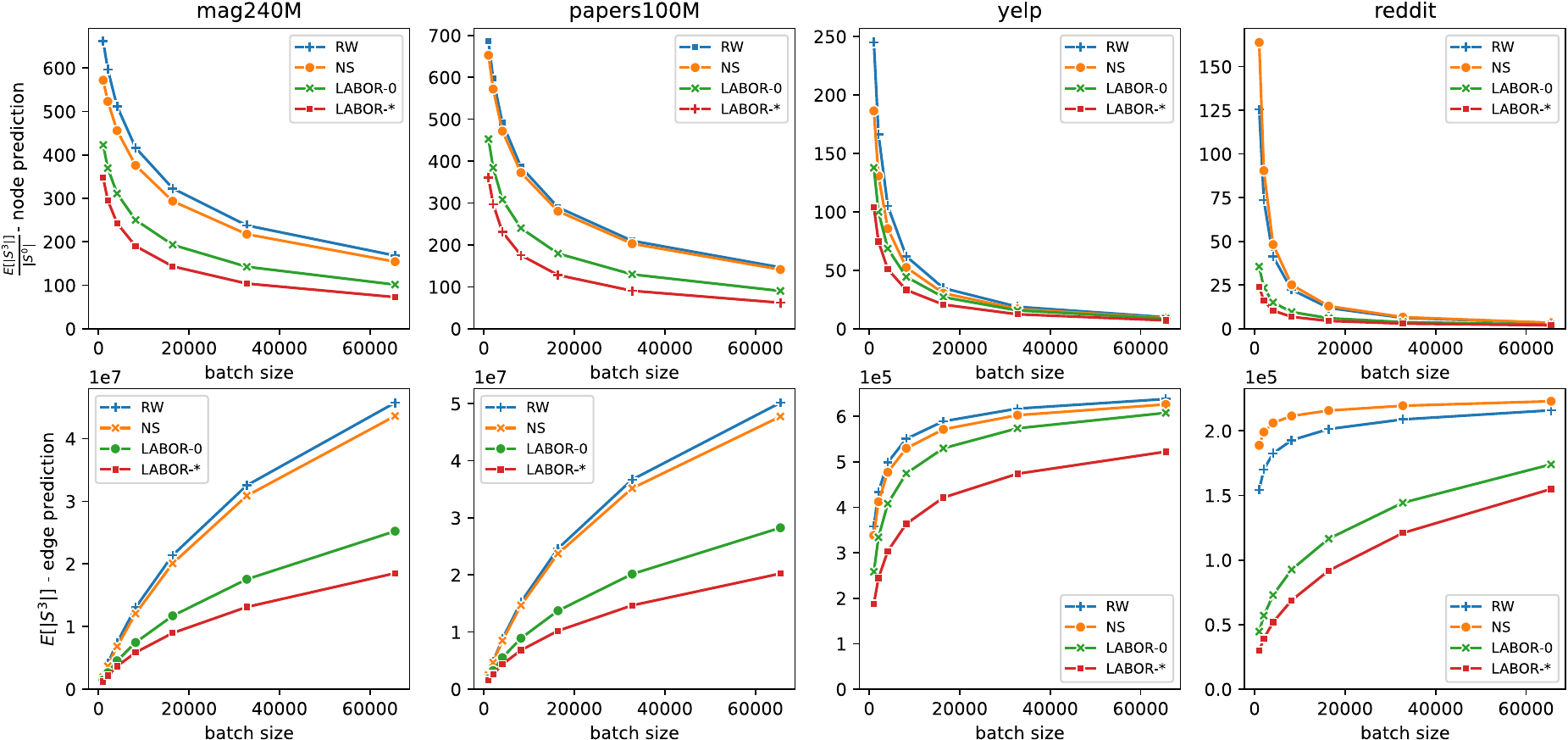}
    \vspace*{-3ex}
    \caption{Monotonicity of the work. x-axis shows the batch size, y-axis shows $\frac{E[|S^3|]}{|S^0|}$ (see~\cref{th:work_monotonicity})
    for node prediction (top row) and $E[|S^3|]$ (see~\cref{th:overlap_monotonicity})
    for edge prediction (bottom row), where $E[|S^3|]$ denotes
    the expected number of sampled vertices in the 3rd layer and
    $|S^0|$ denotes the batch size.
    RW stands for Random Walks, NS for Neighbor Sampling,
    and LABOR-0/* for the two different variants of
    the LABOR sampling algorithm described in~\cref{subsecc:graph_sampling}.
    }
    \label{figc:num_input_nodes}
\end{figure*}

We use three sampling approaches, NS, LABOR, and RW, to demonstrate that the work to process an epoch decreases as the batch size increases for the $L=3$ layer case across these three different classes of sampling algorithms.
We carried out our evaluation in two problem settings: node and edge prediction.
For node prediction, a batch of training vertices is sampled
with a given batch size. Then, the graph sampling algorithms described in~\cref{subsecc:graph_sampling_background}
are applied to sample the neighborhood
of this batch. The top row of~\cref{figc:num_input_nodes} shows how many input vertices is required on
average to process an epoch, specifically $\frac{E[|S^3|]}{|S^0|}$.
For edge prediction, we
add reverse edges to the graph making it undirected and
sample a batch of edges.
For each of these edges a random negative edge (an edge that is not part of $E$) with one endpoint
coinciding with the positive edge is sampled.
Then, all of the endpoints of these positive and
negative edges are used as seed vertices to sample their neighborhoods.
The bottom row of~\cref{figc:num_input_nodes} shows $E[|S^3|]$.

We can see that in all use cases, datasets and sampling algorithms, the work to process an epoch is
monotonically decreasing as we proved in~\cref{th:work_monotonicity}. We also see the plot of the expected number of vertices sampled, $E[|S^3|]$, is
concave with respect to batch size as we already know from~\cref{th:overlap_monotonicity}.

Another observation is that the concavity characteristic of $E[|S^3|]$ seems to differ for
different sampling algorithms. In increasing order of concavity we
have RW, NS, LABOR-0 and LABOR-*. The more concave a sampling algorithm's $E[|S^L|]$ curve is,
the less it is affected from the NEP and more savings are available through the use of the proposed
methods in~\cref{subsecc:cooperating_pes,subsecc:dependent_batches}. Note that the
differences would grow with a larger choice of layer count $L$.

\subsection{Dependent Minibatches}
\label{subsecc:dependent_batches_experiment}

\begin{figure*}[ht]
    \centering
    \includegraphics[width=\linewidth]{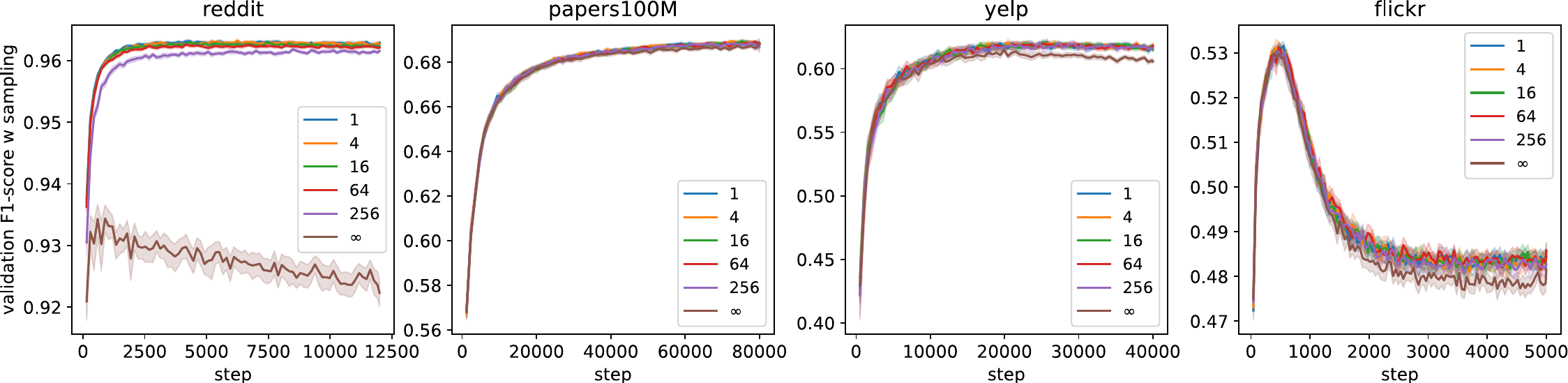}
    \vspace*{-4ex}
    \caption{The validation F1-score with NS sampled neighborhoods trained with the
    LABOR-0 sampling algorithm with $1024$ batch size and varying $\kappa$ dependent minibatches,
    $\kappa=\infty$ denotes infinite dependency, meaning the neighborhood
    sampled for a vertex stays static during training. See~\cref{figc:same_loss_cache_miss} for
    cache miss rates. See~\cref{figc:same_loss_cache_sampling} for
    the training loss and F1-score with the dependent sampler.}
    \label{figc:same_loss_cache}
    \vspace*{-2ex}
\end{figure*}

\begin{table}[ht]
\caption{Test F1-scores at the highest validation F1-score corresponding to~\cref{figc:same_loss_cache}, using early stopping. Averages of 40 runs are presented.}
\begin{small}
\begin{center}
\begin{tabular}{l|ccccc}
\toprule
\textbf{$\bm\kappa$} & \textbf{reddit} & \textbf{papers100M} & \textbf{yelp} & \textbf{flickr} \\
\midrule
1        & 96.72 $\pm$ 0.06 & 66.58 $\pm$ 0.17 & 63.40 $\pm$ 0.18 & 53.82 $\pm$ 0.21 \\
4        & 96.71 $\pm$ 0.05 & 66.60 $\pm$ 0.13 & 63.39 $\pm$ 0.19 & 53.87 $\pm$ 0.18 \\
16       & 96.70 $\pm$ 0.05 & 66.58 $\pm$ 0.20 & 63.42 $\pm$ 0.18 & 53.90 $\pm$ 0.17 \\
64       & 96.67 $\pm$ 0.05 & 66.55 $\pm$ 0.20 & 63.35 $\pm$ 0.18 & 53.88 $\pm$ 0.20 \\
256      & 96.65 $\pm$ 0.06 & 66.56 $\pm$ 0.14 & 63.35 $\pm$ 0.19 & 53.86 $\pm$ 0.24 \\
$\infty$ & 95.01 $\pm$ 0.30 & 66.46 $\pm$ 0.15 & 62.88 $\pm$ 0.18 & 53.84 $\pm$ 0.19 \\
\bottomrule
\end{tabular}
\end{center}
\end{small}
\label{tabc:same_loss_cache}
\vspace*{-2ex}
\end{table}

\begin{figure*}[ht]
    \centering
\begin{subfigure}{.48\linewidth}
    \includegraphics[width=1\linewidth]{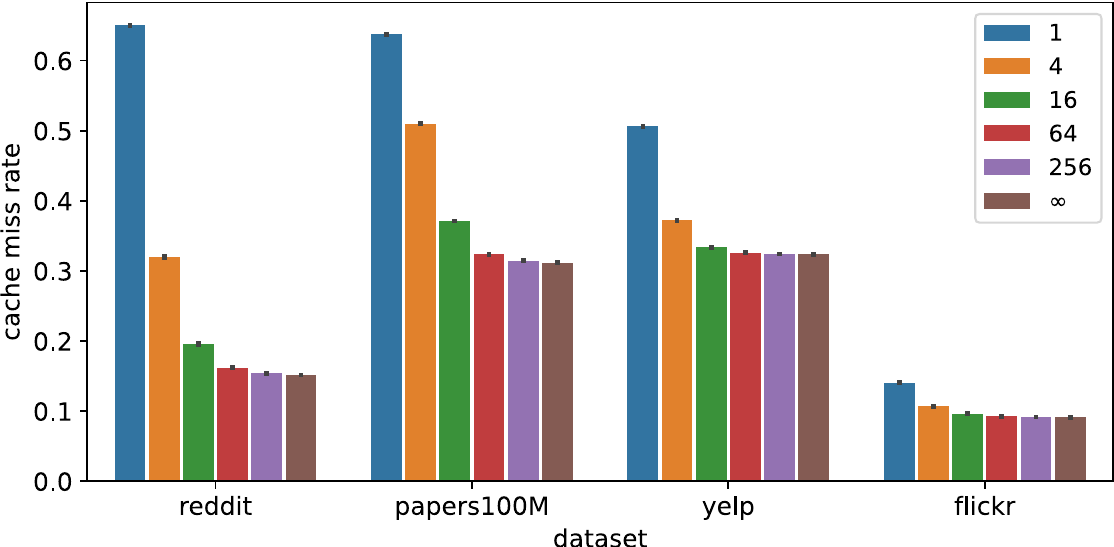}
    \vspace*{-2ex}
    \caption{1 GPU, cache sizes are listed in~\cref{tabc:dataset}.}
    \label{figc:same_loss_cache_miss}
\end{subfigure}
\hspace*{1em}
\begin{subfigure}{.48\linewidth}
    \centering
    \includegraphics[width=1\linewidth]{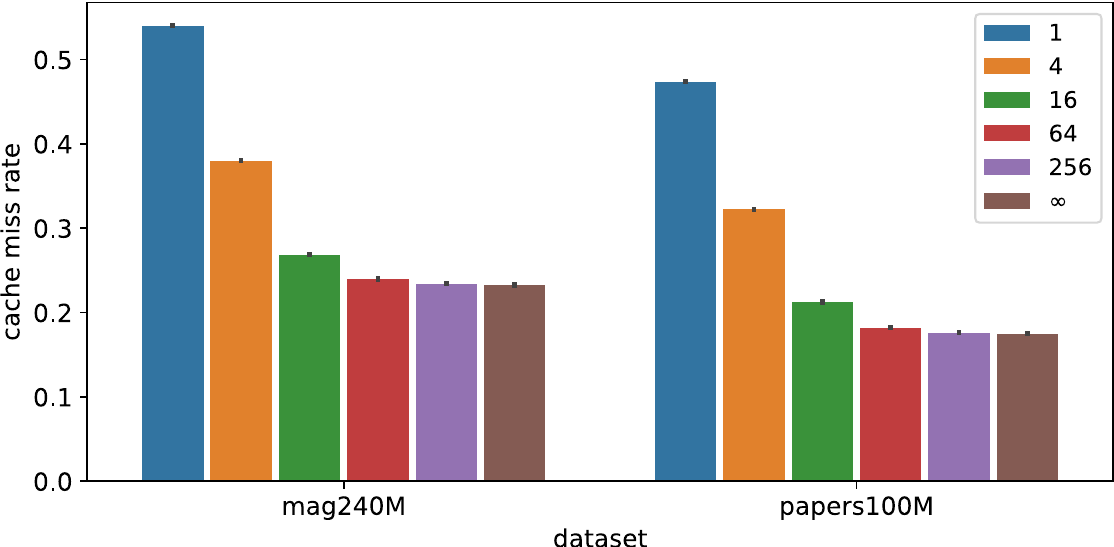}
    \vspace*{-2ex}
    \caption{4 cooperating GPUs, each having a cache of size 1M.}
    \label{figc:coop_dependent_cache_miss}
\end{subfigure}
\caption{LRU-cache miss rates for LABOR-0 sampling algorithm with $1024$
    batch size per GPU and varying $\kappa$ dependent minibatches,
    $\kappa=\infty$ denotes infinite dependency.}
\end{figure*}

We vary the batch dependency parameter $\kappa$ introduced in~\cref{subsecc:dependent_batches}
for the LABOR-0 sampler with a batch size of $1024$. Our
expectation is that as consecutive batches become more dependent on each other, the subgraphs
used during consecutive steps of training would start overlapping with each other, in which
case, the vertex embedding accesses would become more localized. We attempted to capture
this increase in temporal locality in vertex embedding accesses by implementing an LRU cache
to fetch them, the cache sizes used for different datasets is given in~\cref{tabc:dataset}.
Note that the cache miss rate is proportional to the amount of data that needs to be copied
from the vertex embedding storage. The~\cref{figc:same_loss_cache_miss} shows that as $\kappa$ increases,
the cache miss rate across all datasets
drops. On reddit, this is a drop from 64\% to 16\% on, a 4x improvement. We also observe that the
improvement is monotonically increasing as a function of $\frac{|E|}{|V|}$ given in~\cref{tabc:dataset}.
\cref{tabc:same_loss_cache,figc:same_loss_cache} show that training is not negatively affected across all datasets up to
$\kappa = 256$ with less than $0.1\%$ F1-score difference, after which point the validation F1-score with w/o
sampling starts to diverge from the $\kappa = 1$ case. Runtime benefits of this approach can be observed by comparing
the \textbf{Cache} and \textbf{Cache, $\bm\kappa$} columns in~\cref{tabc:cooperative_runtimes}.
\cref{subsecc:dependent_batches_cont} has additional
discussion about the effect of varying $\kappa$ and the last column of~\cref{tabc:dataset} shows the number
of minibatches in an epoch during training.


\subsection{Cooperative Minibatching}
\label{subsecc:cooperative_minibatching_experiment}

\begin{table*}[t!]
    \caption{Cooperative vs independent minibatching runtimes per minibatch (ms) on three different systems with 4 and 8
    NVIDIA A100 80 GB GPUs, and 16 NVIDIA V100 32GB GPUs. I/C denotes whether independent or cooperative minibatching is used. 
    Samp. is short for Graph Sampling, Feature Copy stands 
    for vertex embedding copies over PCI-e and Cache denotes the runtime of copies performed with a cache that can hold $10^6$ 
    vertex embeddings per A100 and $5 \times 10^5$ per V100. $\kappa$ denotes the use of batch dependency $\kappa = 64$.
    F/B means forward/backward. Total time is computed by the fastest available Feature Copy time, the sampling time, and the F/B 
    time. $|S^0|$ is the global batch size and $b$ is the the batch size per GPU. $\alpha$ stands for cross GPU communication 
    bandwidth (NVLink), $\beta$ for PCI-e bandwidth and $\gamma$ for GPU global memory bandwidth. {\color{OliveGreen}Green} was 
    used to indicate the better result between independent and cooperative minibatching, while \textbf{Bold} was used to
    highlight the feature copy time included in the \textbf{Total} column.}
    \vspace*{-2ex}
    \label{tabc:cooperative_runtimes}
    \begin{center}
    \begin{adjustbox}{width=\linewidth,center}
    \begin{tabular}{c | r | r | r | r | r | r | r | r | r}
    \toprule
    \multirowcell{2}{\textbf{\# GPUs, $\bm\gamma$} \\ \textbf{$\bm\alpha, \bm\beta, \bm{|S^0|}$}} & \multirowcell{2}{\textbf{Dataset} \\ \textbf{\& Model}}
    & \multirow{2}{*}{\textbf{Sampler}}
    & \multirow{2}{*}{\textbf{I/C}} & \multirow{2}{*}{\textbf{Samp.}} & \multicolumn{3}{c|}{\textbf{Feature Copy}} & \multirow{2}{*}{\textbf{F/B}} & \multirow{2}{*}{\textbf{Total}} \\
    & & & & & \textbf{-} & \textbf{Cache} & \textbf{Cache, $\bm\kappa$} & &  \\
    \midrule
    \multirowcell{8}{4 A100 \\ $\gamma =$ 2TB/s \\ $\alpha =$ 600GB/s \\ $\beta =$ 64GB/s \\ $|S^0| = 2^{12}$ \\ $b=1024$} & \multirowcell{4}{papers100M \\ GCN} & \multirowcell{2}{LABOR-0} & Indep & 21.7 & 18.4 & 16.8 & \textbf{11.2} & \color{OliveGreen}{8.9} & 41.8 \\
     & & & Coop & \color{OliveGreen}{17.7} & 14.0 & \color{OliveGreen}{10.1} & \color{OliveGreen}{\textbf{5.8}} & 13.0 & \color{OliveGreen}{\textbf{36.5}} \\ \cline{3-10}
     & & \multirowcell{2}{NS} & Indep & 16.1 & 26.5 & \textbf{22.1} & - & \color{OliveGreen}{10.1} & 48.3 \\
     & & & Coop & \color{OliveGreen}{11.9} & 21.3 & \color{OliveGreen}{\textbf{12.9}} & - & 15.0 & \color{OliveGreen}{\textbf{39.8}} \\ \cline{2-10}
     & \multirowcell{4}{mag240M \\ R-GCN} & \multirowcell{2}{LABOR-0} & Indep & 26.0 & 57.9 & 56.0 & \textbf{41.0} & 199.9 & 266.9 \\
     & & & Coop & \color{OliveGreen}{20.0} & 51.1 & \color{OliveGreen}{36.9} & \color{OliveGreen}{\textbf{23.4}} & \color{OliveGreen}{183.3} & \color{OliveGreen}{\textbf{226.7}} \\ \cline{3-10}
     &  & \multirowcell{2}{NS} & Indep & 14.4 & 78.0 & \textbf{71.2} & - & 223.0 & 308.6 \\
     & & & Coop & \color{OliveGreen}{12.3} & 73.9 & \color{OliveGreen}{\textbf{47.5}} & - & \color{OliveGreen}{215.6} & \color{OliveGreen}{\textbf{275.4}} \\
    \midrule
    \multirowcell{8}{8 A100 \\ $\gamma =$ 2TB/s \\ $\alpha =$ 600GB/s \\ $\beta =$ 64GB/s \\ $|S^0| = 2^{13}$ \\ $b=1024$} & \multirowcell{4}{papers100M \\ GCN} & \multirowcell{2}{LABOR-0} & Indep & 21.3 & 21.1 & 18.7 & \textbf{12.0} & \color{OliveGreen}{9.3} & 42.6 \\
     & & & Coop & \color{OliveGreen}{16.5} & 12.4 & \color{OliveGreen}{7.1} & \color{OliveGreen}{\textbf{4.0}} & 13.5 & \color{OliveGreen}{\textbf{34.0}} \\ \cline{3-10}
     & & \multirowcell{2}{NS} & Indep & 15.8 & 31.0 & \textbf{24.5} & - & \color{OliveGreen}{10.3} & 50.6 \\
     & & & Coop & \color{OliveGreen}{12.5} & 19.4 & \color{OliveGreen}{\textbf{9.0}} & - & 15.6 & \color{OliveGreen}{\textbf{37.1}} \\ \cline{2-10}
     & \multirowcell{4}{mag240M \\ R-GCN} & \multirowcell{2}{LABOR-0} & Indep & 30.6 & 70.1 & 66.2 & \textbf{46.8} & 202.1 & 279.5 \\
     & & & Coop & \color{OliveGreen}{21.6} & 50.6 & \color{OliveGreen}{29.0} & \color{OliveGreen}{\textbf{19.3}} & \color{OliveGreen}{172.2} & \color{OliveGreen}{\textbf{213.1}} \\ \cline{3-10}
     & & \multirowcell{2}{NS} & Indep & 15.0 & 94.9 & \textbf{80.9} & - & 224.8 & 320.7 \\
     & & & Coop & 14.9 & 71.6 & \color{OliveGreen}{\textbf{39.6}} & - & \color{OliveGreen}{209.0} & \color{OliveGreen}{\textbf{263.5}} \\
    \midrule
    \multirowcell{8}{16 V100 \\ $\gamma =$ 0.9TB/s \\ $\alpha =$ 300GB/s \\ $\beta =$ 32GB/s \\ $|S^0| = 2^{13}$ \\ $b=512$} & \multirowcell{4}{papers100M \\ GCN} & \multirowcell{2}{LABOR-0} & Indep & 39.1 & 44.5 & 40.2 & \textbf{29.4} & \color{OliveGreen}{15.1} & 83.6 \\
     & & & Coop & \color{OliveGreen}{26.9} & 22.7 & \color{OliveGreen}{10.4} & \color{OliveGreen}{\textbf{4.9}} & 19.1 & \color{OliveGreen}{\textbf{50.9}} \\ \cline{3-10}
     & & \multirowcell{2}{NS} & Indep & \color{OliveGreen}{18.0} & 61.3 & \textbf{52.0} & - & \color{OliveGreen}{16.2} & 86.2 \\
     & & & Coop & 19.2 & 34.9 & \color{OliveGreen}{\textbf{13.0}} & - & 21.3 & \color{OliveGreen}{\textbf{53.5}} \\ \cline{2-10}
     & \multirowcell{4}{mag240M \\ R-GCN} & \multirowcell{2}{LABOR-0} & Indep & 50.8 & 128.8 & 121.3 & \textbf{96.2} & 156.1 & 303.1 \\
     & & & Coop & \color{OliveGreen}{29.2} & 78.1 & \color{OliveGreen}{42.8} & \color{OliveGreen}{\textbf{23.5}} & \color{OliveGreen}{133.3} & \color{OliveGreen}{\textbf{186.0}} \\ \cline{3-10}
     & & \multirowcell{2}{NS} & Indep & 19.3 & 167.3 & \textbf{152.6} & - & 170.9 & 342.8 \\
     & & & Coop & 19.3 & 116.1 & \color{OliveGreen}{\textbf{53.1}} & - & \color{OliveGreen}{160.4} & \color{OliveGreen}{\textbf{232.8}} \\
    \bottomrule
    \end{tabular}
    \end{adjustbox}
    \end{center}
    \vspace*{-2ex}
\end{table*}

We use our largest datasets, mag240M and papers100M, as distributed training
is motivated by large-scale datasets.
We present our runtime results on systems equipped with NVIDIA GPUs, with 4 and 8 A100 80 GB~\citep{a100} and
16 V100 32GB~\citep{v100}, all with NVLink
interconnect between the GPUs (600 GB/s for A100 and 300 GB/s for V100). The GPUs perform all stages of GNN
training and the CPUs are only used to launch kernels for the GPUs. Feature copies are performed by GPUs as well,
accessing pinned feature tensors over the PCI-e using zero-copy access.
In cooperative minibatching, both data size and computational cost are shrinking with increasing numbers of GPUs,
relative to independent minibatching. We
use the GCN model for papers100M and the R-GCN model~\citep{rgcn} for mag240M.
As seen in~\cref{tabc:cooperative_runtimes},
cooperative minibatching reduces all the runtimes
for different stages of GNN training, except for the F/B (forward/backward) times on papers100M
where the computational cost is not high enough to hide the overhead of communication.

\begin{table}[ht!]
    \centering
    \begin{small}
    \caption{Runtime improvements of Cooperative Minibatching over Independent Minibatching compiled
    from the \textbf{Total} column of~\cref{tabc:cooperative_runtimes}. This is a further improvement on top of the speedup
    of independent minibatching already achieves using the same number GPUs.}
    \label{tabc:cooperative_speedup}
    \begin{tabular}{c | c | r | r | r}
    \toprule
    \textbf{Dataset \& Model}
    & \textbf{Sampler}
     & \textbf{4 GPUs} & \textbf{8 GPUs} & \textbf{16 GPUs} \\
    \midrule
    \multirowcell{2}{papers100M \& GCN} & LABOR-0 & $15\%$ & $25\%$ & $64\%$ \\
     & NS & $21\%$ & $36\%$ & $61\%$ \\ \hline
    \multirowcell{2}{mag240M \& R-GCN} & LABOR-0 & $18\%$ & $31\%$ & $63\%$ \\
     & NS & $12\%$ & $22\%$ & $47\%$ \\
    \bottomrule
    \end{tabular}
    \end{small}
\end{table}

If we take the numbers in the \textbf{Total} columns from~\cref{tabc:cooperative_runtimes}, divide independent
runtimes by the corresponding cooperative ones, then we get~\cref{tabc:cooperative_speedup}. We can see that
the theoretical decrease in work results in increasing speedup numbers with the increasing number of PEs, due 
to~\cref{th:work_monotonicity}.
We would like to point out that $\frac{E[|S^3|]}{|S^0|}$ curves in~\cref{figc:num_input_nodes} are responsible
for these results. With $P$ PEs and $|S^0|$ global batch size, the work performed by independent minibatching vs
cooperative minibatching can be compared by looking at $x=\frac{1}{P}|S^0|$ vs $x=|S^0|$ respectively.

We also ran experiments that show that graph partitioning using METIS~\citep{metis98} prior to the start of training
can help the scenarios where communication overhead is significant. The forward-backward time
decreases from 13.0ms to 12.0ms on papers100M with LABOR-0 on 4 NVIDIA A100 GPUs with partitioning due to 
reduced communication overhead using the same setup as~\cref{tabc:cooperative_runtimes}.

Increasing the number of GPUs increases the advantage of cooperative minibatching compared to independent minibatching. 
The forward-backward time on mag240M with LABOR-0 is 200 (same as independent baseline), 194, 187 and 183 ms with 1, 2, 3 and 4 
cooperating GPUs, respectively measured on the NVIDIA DGX Station A100 machine with R-GCN. When using a different GNN model such as GATs~\citep{velickovic2018graph}, the forward-backward runtime is 190 ms on 1 GPU vs 172 ms on 4 GPUs on mag240M.
The decrease in runtime with increasingly cooperating GPUs is due to the decrease
in redundant work they have to perform. Even though the batch size per GPU is constant, the runtime
goes down similar to the plots in the top row of ~\cref{figc:num_input_nodes}, except that it follows,
$\frac{kE[|S^2|]}{|S^0|}$, the average number of edges in the 3rd layer when a sampler with
fanout $k$ is used.

Additionally, we demonstrate that there is no significant model convergence difference
between independent vs cooperative minibatching in~\cref{subsecc:coop_indep_conv}.

\subsubsection{Cooperative-Dependent Minibatching}
\label{subsecc:cooperative_dependent_batches}

\begin{table}[h!]
    \centering
    \begin{small}
    \caption{Runtime improvements of Dependent Minibatching, for Independent and Cooperative Minibatching methods, compiled
    from the \textbf{Cache, $\bm\kappa$} and \textbf{Cache} columns of~\cref{tabc:cooperative_runtimes} with LABOR-0. Making consecutive minibatches
    dependent increases temporal locality, hence reducing cache misses.
    }
    \label{tabc:cooperative_dependent_speedup}
    \begin{tabular}{c | c | r | r | r}
    \toprule
    \textbf{Dataset \& Model}
    & \textbf{I/C}
     & \textbf{4 GPUs} & \textbf{8 GPUs} & \textbf{16 GPUs} \\
    \midrule
    \multirowcell{2}{papers100M \& GCN} & Indep + Depend & $50\%$ & $57\%$ & $37\%$ \\
                                        & Coop  + Depend & $74\%$ & $78\%$ & $112\%$ \\ \hline 
    \multirowcell{2}{mag240M \& R-GCN} & Indep + Depend & $37\%$ & $41\%$ & $26\%$ \\
                                       & Coop + Depend & $58\%$ & $50\%$ & $82\%$ \\
    \bottomrule
    \end{tabular}
    \end{small}
\end{table}

We use the same experimental setup as~\cref{subsecc:cooperative_minibatching_experiment} but
vary the $\kappa$ parameter to show that cooperative minibatching can be used with
dependent batches (\cref{figc:coop_dependent_cache_miss}). We use a cache size of 1M vertex embeddings per GPU.
Cooperative feature loading effectively increases the global cache size since each PE (GPU) caches
only the vertices assigned to them while independent feature loading can have duplicate entries
across caches. For our largest dataset mag240M,
on top of $1.4\times$ reduced work due to cooperative minibatching alone,
the cache miss rates were reduced by more than $2\times$, making the total improvement $3\times$.
Runtime results for $\kappa \in \{1, 256\}$ are presented in~\cref{tabc:cooperative_runtimes},
the Feature Copy \textbf{Cache} and \textbf{Cache, $\bm\kappa$} columns.~\cref{tabc:cooperative_dependent_speedup}
summarizes these results by dividing the runtimes in \textbf{Cache} by \textbf{Cache, $\bm\kappa$} and reporting
percentage improvements.

\section{Key insights of Cooperative Minibatching}

In~\cref{secc:background}, the work $W(|S^0|)$ to process an epoch (full pass over the 
dataset) for a given minibatch size $|S^0|$ is characterized by the number of minibatches 
in an epoch ($\frac{|V|}{|S^0|}$) $\times$ the amount of work to process a single 
minibatch, which is approximated by the sum of the number of sampled vertices in 
each layer ($\sum_{l=1}^L |S^l|$). This can be seen in~\cref{eqc:work_definition}.

\cref{eqc:work_definition} only considers the number of processed vertices and it is good 
enough for our purposes. Since all the sampling algorithms we consider 
in~\cref{subsecc:graph_sampling} have fanout parameters $k$, the number of edges sampled 
for each seed vertex has an upper bound $k$. So, given vertices $S^l$ for the $l$th layer, 
the number of sampled edges in that same layer will be $\leq k |S^l|$. Clearly for almost
any GNN model, the runtime complexity to process layer $l$ is linearly increasing w.r.t.
the number of vertices ($|S^l|$) and edges ($\leq k|S^l|$) processed, so the runtime 
complexity is $\mathcal{O}(|S^l| + k|S^l|) = \mathcal{O}(|S^l|)$.

A more comprehensive analysis of the runtime complexities of Cooperative and 
Independent Minibatching approaches is provided in~\cref{subsecc:complexity_analysis}, 
taking into account the exact numbers of sampled vertices ($|S^l|$), edges ($|E^l|$), 
and various communication bandwidths ($\alpha$, $\gamma$, $\beta$) and even graph 
partition quality $c$ and cache misses $\rho$.

As Cooperative Minibatching considers a single minibatch of size $B$ for all $P$ PEs, the 
growth of the number of sampled vertices and edges is characterized by $B$ as the 
minibatch size. In contrast, Independent Minibatching assigns different minibatches of 
sizes $\frac{B}{P}$ to each PE, so the growth of the sampled vertices and edges is 
characterized by $\frac{B}{P}$ as the minibatch size. Due 
to~\cref{th:work_monotonicity,th:overlap_monotonicity}, the work $W(B)$ with respect to a 
minibatch size $B$ is a concave function, we have $W(B) \leq PW(\frac{B}{P})$, meaning 
that Cooperative Minibatching is theoretically faster than Independent Minibatching, if 
communication between PEs is infinitely fast. For empirical data, one can look at the 
first-row of~\cref{figc:num_input_nodes} at the $x$-axis $B$ for Cooperative and $\frac{B}{P}$ for Independent, due to:
\[
W(B) \leq PW(\frac{B}{P}) \iff \frac{W(B)}{B} \leq \frac{W(\frac{B}{P})}{\frac{B}{P}}
\]
as proven by~\cref{th:work_monotonicity}. As $\frac{E[|S^3|]}{|S^0|}$ curves
in~\cref{figc:num_input_nodes} are decreasing, the work of Cooperative Minibatching is 
significantly less than Independent Minibatching. Finally, modern multi-GPU computer 
systems have very fast inter-GPU communication bandwidths, e.g. NVLink, which is why we 
are able to show favorable results compared to Independent Minibatching despite performing 
seemingly unnecessary communication.

\section{Conclusions}

In this paper, we investigated the difference between DNN and GNN minibatch training.
We observed that the cost of processing a minibatch is a concave function of
batch size in GNNs, unlike DNNs where the cost scales linearly.
We then presented theorems that this is indeed the case for every graph and
then proceeded to propose two approaches to take advantage of cost concavity.
The first approach, which we call cooperative minibatching
proposes to partition a minibatch between multiple PEs and process it cooperatively. This is
in contrast to existing practice of having independent minibatches per PE, and avoids duplicate
work that is a result of vertex and edge repetition across PEs. The second approach proposes
the use of consecutive dependent minibatches, through which the temporal locality of vertex
and edge accesses is manipulated. As batches get more dependent, the locality increases.
We demonstrated this increase in locality by employing an LRU-cache for vertex embeddings on GPUs.
Finally, we showed that these approaches can be combined without affecting convergence,
and speed up multi-GPU GNN training by up to $64\%$.





\bibliography{main}

\begin{thebibliography}{50}
\providecommand{\natexlab}[1]{#1}
\providecommand{\url}[1]{\texttt{#1}}
\expandafter\ifx\csname urlstyle\endcsname\relax
  \providecommand{\doi}[1]{doi: #1}\else
  \providecommand{\doi}{doi: \begingroup \urlstyle{rm}\Url}\fi

\bibitem[Allen-Zhu \& Hazan(2016)Allen-Zhu and Hazan]{sgd_vs_gd}
Zeyuan Allen-Zhu and Elad Hazan.
\newblock Variance reduction for faster non-convex optimization.
\newblock In \emph{Proceedings of The 33rd International Conference on Machine Learning}, pp.\  699--707. PMLR, 20--22 Jun 2016.
\newblock URL \url{https://proceedings.mlr.press/v48/allen-zhua16.html}.

\bibitem[Artico et~al.(2020)Artico, Smolyarenko, Vinciotti, and Wit]{powerlawgraphs}
I.~Artico, I.~Smolyarenko, V.~Vinciotti, and E.~C. Wit.
\newblock How rare are power-law networks really?
\newblock In \emph{Royal Society}, volume 476, 2020.
\newblock URL \url{http://doi.org/10.1098/rspa.2019.0742}.

\bibitem[Balin \& \c{C}ataly\"{u}rek(2023)Balin and \c{C}ataly\"{u}rek]{Balin23-NeurIPS}
Muhammed~Fatih Balin and \"{U}mit \c{C}ataly\"{u}rek.
\newblock Layer-neighbor sampling --- defusing neighborhood explosion in gnns.
\newblock In \emph{Advances in Neural Information Processing Systems}, volume~36, 2023.
\newblock URL \url{https://proceedings.neurips.cc/paper_files/paper/2023/file/51f9036d5e7ae822da8f6d4adda1fb39-Paper-Conference.pdf}.
\newblock Also available as a Tech Rep on {ArXiv}.

\bibitem[Balin et~al.(2022)Balin, LaSalle, and \c{C}ataly\"{u}rek]{cooperative_minibatching_aug_2022}
Muhammed~Fatih Balin, Dominique LaSalle, and \"Umit~V. \c{C}ataly\"{u}rek.
\newblock Cooperative minibatching, August 2022.
\newblock URL \url{https://github.com/dmlc/dgl/pull/4337}.

\bibitem[Balin et~al.(2023)Balin, LaSalle, and \c{C}ataly\"{u}rek]{cooperative_minibatching_jan_2023}
Muhammed~Fatih Balin, Dominique LaSalle, and \"Umit~V. \c{C}ataly\"{u}rek.
\newblock Cooperative minibatching in gnns, Jan 2023.
\newblock URL \url{https://github.com/mfbalin/mfbalin.github.io/blob/main/public/papers/6621_Cooperative_Minibatching_.pdf}.
\newblock ICML Submission \#6621. (Rejected).

\bibitem[Bertsekas(1994)]{dnnminibatch}
D.P. Bertsekas.
\newblock Incremental least squares methods and the extended kalman filter.
\newblock In \emph{Proceedings of 1994 33rd IEEE Conference on Decision and Control}, volume~2, pp.\  1211--1214 vol.2, 1994.
\newblock \doi{10.1109/CDC.1994.411166}.

\bibitem[Cai et~al.(2023)Cai, Zhou, Yan, Zheng, Song, Zheng, Cheng, and Karypis]{dsp2023}
Zhenkun Cai, Qihui Zhou, Xiao Yan, Da~Zheng, Xiang Song, Chenguang Zheng, James Cheng, and George Karypis.
\newblock Dsp: Efficient gnn training with multiple gpus.
\newblock In \emph{Proceedings of the 28th ACM SIGPLAN Annual Symposium on Principles and Practice of Parallel Programming}, PPoPP '23, pp.\  392–404, 2023.
\newblock \doi{10.1145/3572848.3577528}.

\bibitem[Chiang et~al.(2019)Chiang, Li, Liu, Bengio, Si, and Hsieh]{Chiang2019}
Wei~Lin Chiang, Yang Li, Xuanqing Liu, Samy Bengio, Si~Si, and Cho~Jui Hsieh.
\newblock {Cluster-GCN: An efficient algorithm for training deep and large graph convolutional networks}.
\newblock In \emph{Proceedings of the ACM SIGKDD International Conference on Knowledge Discovery and Data Mining}, pp.\  257--266. Association for Computing Machinery, jul 2019.
\newblock \doi{10.1145/3292500.3330925}.

\bibitem[Ching et~al.(2015)Ching, Edunov, Kabiljo, Logothetis, and Muthukrishnan]{trillionfacebook}
Avery Ching, Sergey Edunov, Maja Kabiljo, Dionysios Logothetis, and Sambavi Muthukrishnan.
\newblock One trillion edges: Graph processing at facebook-scale.
\newblock \emph{Proc. VLDB Endow.}, 8\penalty0 (12):\penalty0 1804–1815, aug 2015.
\newblock \doi{10.14778/2824032.2824077}.

\bibitem[Dean et~al.(2012)Dean, Corrado, Monga, Chen, Devin, Mao, Ranzato, Senior, Tucker, Yang, Le, and Ng]{intralayer}
Jeffrey Dean, Greg Corrado, Rajat Monga, Kai Chen, Matthieu Devin, Mark Mao, Marc\textquotesingle~aurelio Ranzato, Andrew Senior, Paul Tucker, Ke~Yang, Quoc Le, and Andrew Ng.
\newblock Large scale distributed deep networks.
\newblock In F.~Pereira, C.J. Burges, L.~Bottou, and K.Q. Weinberger (eds.), \emph{Advances in Neural Information Processing Systems}, volume~25. Curran Associates, Inc., 2012.
\newblock URL \url{https://proceedings.neurips.cc/paper/2012/file/6aca97005c68f1206823815f66102863-Paper.pdf}.

\bibitem[Fey et~al.(2021)Fey, Lenssen, Weichert, and Leskovec]{pmlr-v139-fey21a}
Matthias Fey, Jan~E. Lenssen, Frank Weichert, and Jure Leskovec.
\newblock Gnnautoscale: Scalable and expressive graph neural networks via historical embeddings.
\newblock In Marina Meila and Tong Zhang (eds.), \emph{Proceedings of the 38th International Conference on Machine Learning}, volume 139 of \emph{Proceedings of Machine Learning Research}, pp.\  3294--3304. PMLR, 18--24 Jul 2021.
\newblock URL \url{https://proceedings.mlr.press/v139/fey21a.html}.

\bibitem[Gandhi \& Iyer(2021)Gandhi and Iyer]{gandhi2021p3}
Swapnil Gandhi and Anand~Padmanabha Iyer.
\newblock P3: Distributed deep graph learning at scale.
\newblock In \emph{15th USENIX Symposium on Operating Systems Design and Implementation (OSDI 21)}, pp.\  551--568, 2021.

\bibitem[Ginsburg et~al.(2017)Ginsburg, Gitman, and You]{ginsburg2018large}
Boris Ginsburg, Igor Gitman, and Yang You.
\newblock Large batch training of convolutional networks with layer-wise adaptive rate scaling.
\newblock Technical Report arXiv:1708.03888, ArXiv, September 2017.
\newblock URL \url{http://arxiv.org/abs/1708.03888}.

\bibitem[Goyal et~al.(2018)Goyal, Dollár, Girshick, Noordhuis, Wesolowski, Kyrola, Tulloch, Jia, and He]{imagenetonehour}
Priya Goyal, Piotr Dollár, Ross Girshick, Pieter Noordhuis, Lukasz Wesolowski, Aapo Kyrola, Andrew Tulloch, Yangqing Jia, and Kaiming He.
\newblock Accurate, large minibatch sgd: Training imagenet in 1 hour.
\newblock Technical Report arXiv:1706.02677, ArXiv, April 2018.
\newblock URL \url{http://arxiv.org/abs/1706.02677}.

\bibitem[Hamilton et~al.(2017)Hamilton, Ying, and Leskovec]{neighborsampling}
William~L. Hamilton, Rex Ying, and Jure Leskovec.
\newblock Inductive representation learning on large graphs.
\newblock In \emph{Proceedings of the 31st International Conference on Neural Information Processing Systems}, NIPS'17, pp.\  1025–1035, 2017.

\bibitem[Hu et~al.(2020{\natexlab{a}})Hu, Fey, Zitnik, Dong, Ren, Liu, Catasta, and Leskovec]{ogb2020}
Weihua Hu, Matthias Fey, Marinka Zitnik, Yuxiao Dong, Hongyu Ren, Bowen Liu, Michele Catasta, and Jure Leskovec.
\newblock {Open graph benchmark: Datasets for machine learning on graphs}.
\newblock \emph{Advances in Neural Information Processing Systems}, 2020-Decem\penalty0 (NeurIPS):\penalty0 1--34, 2020{\natexlab{a}}.

\bibitem[Hu et~al.(2021)Hu, Fey, Ren, Nakata, Dong, and Leskovec]{ogblsc}
Weihua Hu, Matthias Fey, Hongyu Ren, Maho Nakata, Yuxiao Dong, and Jure Leskovec.
\newblock Ogb-lsc: A large-scale challenge for machine learning on graphs, 2021.
\newblock URL \url{https://arxiv.org/abs/2103.09430}.

\bibitem[Hu et~al.(2020{\natexlab{b}})Hu, Dong, Wang, and Sun]{Hu2020a}
Ziniu Hu, Yuxiao Dong, Kuansan Wang, and Yizhou Sun.
\newblock {Heterogeneous Graph Transformer}.
\newblock \emph{The Web Conference 2020 - Proceedings of the World Wide Web Conference, WWW 2020}, pp.\  2704--2710, 2020{\natexlab{b}}.
\newblock \doi{10.1145/3366423.3380027}.

\bibitem[Jia et~al.(2020)Jia, Lin, Ying, You, Leskovec, and Aiken]{jia2020}
Zhihao Jia, Sina Lin, Rex Ying, Jiaxuan You, Jure Leskovec, and Alex Aiken.
\newblock Redundancy-free computation for graph neural networks.
\newblock In \emph{Proceedings of the 26th ACM SIGKDD International Conference on Knowledge Discovery \& Data Mining}, KDD '20, pp.\  997–1005. Association for Computing Machinery, 2020.
\newblock URL \url{https://doi.org/10.1145/3394486.3403142}.

\bibitem[Jiang \& Luo(2022)Jiang and Luo]{trafficforecastingsurvey}
Weiwei Jiang and Jiayun Luo.
\newblock Graph neural network for traffic forecasting: A survey.
\newblock \emph{Expert Systems with Applications}, 207:\penalty0 117921, nov 2022.
\newblock \doi{10.1016/j.eswa.2022.117921}.

\bibitem[Karypis \& Kumar(1998)Karypis and Kumar]{metis98}
George Karypis and Vipin Kumar.
\newblock A fast and high quality multilevel scheme for partitioning irregular graphs.
\newblock \emph{SIAM Journal on Scientific Computing}, 20\penalty0 (1):\penalty0 359--392, 1998.
\newblock \doi{10.1137/S1064827595287997}.

\bibitem[Keskar et~al.(2016)Keskar, Mudigere, Nocedal, Smelyanskiy, and Tang]{keskar2016large}
Nitish~Shirish Keskar, Dheevatsa Mudigere, Jorge Nocedal, Mikhail Smelyanskiy, and Ping Tak~Peter Tang.
\newblock On large-batch training for deep learning: Generalization gap and sharp minima.
\newblock \emph{arXiv preprint arXiv:1609.04836}, 2016.

\bibitem[Kingma \& Ba(2014)Kingma and Ba]{Kingma2014}
Diederik~P. Kingma and Jimmy Ba.
\newblock Adam: A method for stochastic optimization, 2014.
\newblock URL \url{https://arxiv.org/abs/1412.6980}.

\bibitem[Li et~al.(2014)Li, Zhang, Chen, and Smola]{efficient_minibatch_training}
Mu~Li, Tong Zhang, Yuqiang Chen, and Alexander~J. Smola.
\newblock Efficient mini-batch training for stochastic optimization.
\newblock In \emph{Proceedings of the 20th ACM SIGKDD International Conference on Knowledge Discovery and Data Mining}, KDD '14, pp.\  661–670, 2014.
\newblock \doi{10.1145/2623330.2623612}.

\bibitem[Lin et~al.(2020)Lin, Li, Miao, Liu, and Xu]{pagraph}
Zhiqi Lin, Cheng Li, Youshan Miao, Yunxin Liu, and Yinlong Xu.
\newblock Pagraph: Scaling gnn training on large graphs via computation-aware caching.
\newblock In \emph{Proceedings of the 11th ACM Symposium on Cloud Computing}, SoCC '20, pp.\  401–415, 2020.
\newblock \doi{10.1145/3419111.3421281}.

\bibitem[Liu et~al.(2022)Liu, Sun, and Zhang]{fraudgnn2}
Yajing Liu, Zhengya Sun, and Wensheng Zhang.
\newblock Improving fraud detection via hierarchical attention-based graph neural network, 2022.

\bibitem[Narayanan et~al.(2019)Narayanan, Harlap, Phanishayee, Seshadri, Devanur, Ganger, Gibbons, and Zaharia]{pipelining}
Deepak Narayanan, Aaron Harlap, Amar Phanishayee, Vivek Seshadri, Nikhil~R. Devanur, Gregory~R. Ganger, Phillip~B. Gibbons, and Matei Zaharia.
\newblock Pipedream: Generalized pipeline parallelism for dnn training.
\newblock In \emph{Proceedings of the 27th ACM Symposium on Operating Systems Principles}, SOSP '19, pp.\  1–15, 2019.
\newblock \doi{10.1145/3341301.3359646}.

\bibitem[{NVIDIA}(2020{\natexlab{a}})]{dgxa100}
{NVIDIA}.
\newblock {NVIDIA} {DGX} {A100} system architecture: The universal system for {AI} infrastructure.
\newblock Technical report, {NVIDIA} Corporation, July 2020{\natexlab{a}}.
\newblock URL \url{https://www.nvidia.com/content/dam/en-zz/Solutions/Data-Center/nvidia-dgx-a100-datasheet.pdf}.

\bibitem[{NVIDIA}(2020{\natexlab{b}})]{v100}
{NVIDIA}.
\newblock {NVIDIA} {V100} | datasheet.
\newblock Technical report, {NVIDIA} Corporation, January 2020{\natexlab{b}}.
\newblock URL \url{https://images.nvidia.com/content/technologies/volta/pdf/volta-v100-datasheet-update-us-1165301-r5.pdf}.

\bibitem[{NVIDIA}(2021)]{a100}
{NVIDIA}.
\newblock {NVIDIA} {A100} tensor core {GPU} | datasheet.
\newblock Technical report, {NVIDIA} Corporation, June 2021.
\newblock URL \url{https://www.nvidia.com/content/dam/en-zz/Solutions/Data-Center/a100/pdf/nvidia-a100-datasheet-us-nvidia-1758950-r4-web.pdf}.

\bibitem[{NVIDIA}(2023)]{dgxh100}
{NVIDIA}.
\newblock {NVIDIA} {DGX} {H100}: The gold standard for {AI} infrastructure.
\newblock Technical report, {NVIDIA} Corporation, March 2023.
\newblock URL \url{https://resources.nvidia.com/en-us-dgx-systems/ai-enterprise-dgx}.

\bibitem[{NVIDIA}(2024)]{nvl72}
{NVIDIA}.
\newblock {NVIDIA} {GB200} {NVL72}: Powering the new era of computing.
\newblock Technical report, {NVIDIA} Corporation, March 2024.
\newblock URL \url{https://www.nvidia.com/en-us/data-center/gb200-nvl72/}.

\bibitem[Park et~al.(2022)Park, Min, and Lee]{ginex2022}
Yeonhong Park, Sunhong Min, and Jae~W. Lee.
\newblock Ginex: Ssd-enabled billion-scale graph neural network training on a single machine via provably optimal in-memory caching.
\newblock \emph{Proc. VLDB Endow.}, 15\penalty0 (11):\penalty0 2626–2639, jul 2022.
\newblock ISSN 2150-8097.
\newblock \doi{10.14778/3551793.3551819}.

\bibitem[Paszke et~al.(2019)Paszke, Gross, Massa, Lerer, Bradbury, Chanan, Killeen, Lin, Gimelshein, Antiga, Desmaison, K{\"o}pf, Yang, DeVito, Raison, Tejani, Chilamkurthy, Steiner, Fang, Bai, and Chintala]{Paszke2019PyTorchAI}
Adam Paszke, Sam Gross, Francisco Massa, Adam Lerer, James Bradbury, Gregory Chanan, Trevor Killeen, Zeming Lin, Natalia Gimelshein, Luca Antiga, Alban Desmaison, Andreas K{\"o}pf, Edward Yang, Zach DeVito, Martin Raison, Alykhan Tejani, Sasank Chilamkurthy, Benoit Steiner, Lu~Fang, Junjie Bai, and Soumith Chintala.
\newblock Pytorch: An imperative style, high-performance deep learning library.
\newblock In \emph{NeurIPS}, 2019.

\bibitem[Patel et~al.(2022)Patel, Rajasegarar, Pan, Liu, and Zhu]{fraudgnn}
Vatsal Patel, Sutharshan Rajasegarar, Lei Pan, Jiajun Liu, and Liming Zhu.
\newblock Evangcn: Evolving graph deep neural network based anomaly detection in blockchain.
\newblock In Weitong Chen, Lina Yao, Taotao Cai, Shirui Pan, Tao Shen, and Xue Li (eds.), \emph{Advanced Data Mining and Applications}, pp.\  444--456, 2022.

\bibitem[Polisetty et~al.(2023)Polisetty, Liu, Falus, Fung, Lim, Guan, and Serafini]{polisetty2023gsplit}
Sandeep Polisetty, Juelin Liu, Kobi Falus, Yi~Ren Fung, Seung-Hwan Lim, Hui Guan, and Marco Serafini.
\newblock Gsplit: Scaling graph neural network training on large graphs via split-parallelism, 2023.

\bibitem[Schlichtkrull et~al.(2017)Schlichtkrull, Kipf, Bloem, Berg, Titov, and Welling]{rgcn}
Michael Schlichtkrull, Thomas~N. Kipf, Peter Bloem, Rianne van~den Berg, Ivan Titov, and Max Welling.
\newblock Modeling relational data with graph convolutional networks, 2017.

\bibitem[Shi et~al.(2023)Shi, Liang, and Wang]{shi2023lmc}
Zhihao Shi, Xize Liang, and Jie Wang.
\newblock {LMC}: Fast training of {GNN}s via subgraph sampling with provable convergence.
\newblock In \emph{The Eleventh International Conference on Learning Representations}, 2023.
\newblock URL \url{https://openreview.net/forum?id=5VBBA91N6n}.

\bibitem[Tripathy et~al.(2020)Tripathy, Yelick, and Bulu\c{c}]{CAGNET}
Alok Tripathy, Katherine Yelick, and Ayd\i{}n Bulu\c{c}.
\newblock Reducing communication in graph neural network training.
\newblock In \emph{Proceedings of the International Conference for High Performance Computing, Networking, Storage and Analysis}, 2020.

\bibitem[Veličković et~al.(2018)Veličković, Cucurull, Casanova, Romero, Liò, and Bengio]{velickovic2018graph}
Petar Veličković, Guillem Cucurull, Arantxa Casanova, Adriana Romero, Pietro Liò, and Yoshua Bengio.
\newblock Graph attention networks.
\newblock In \emph{International Conference on Learning Representations}, 2018.
\newblock URL \url{https://openreview.net/forum?id=rJXMpikCZ}.

\bibitem[Waleffe et~al.(2022)Waleffe, Mohoney, Rekatsinas, and Venkataraman]{waleffe2022mariusgnn}
Roger Waleffe, Jason Mohoney, Theodoros Rekatsinas, and Shivaram Venkataraman.
\newblock Mariusgnn: Resource-efficient out-of-core training of graph neural networks, 2022.

\bibitem[Wang et~al.(2019)Wang, Zheng, Ye, Gan, Li, Song, Zhou, Ma, Yu, Gai, Xiao, He, Karypis, Li, and Zhang]{wang2019deep}
Minjie Wang, Da~Zheng, Zihao Ye, Quan Gan, Mufei Li, Xiang Song, Jinjing Zhou, Chao Ma, Lingfan Yu, Yu~Gai, Tianjun Xiao, Tong He, George Karypis, Jinyang Li, and Zheng Zhang.
\newblock Deep graph library: A graph-centric, highly-performant package for graph neural networks, 2019.
\newblock URL \url{https://arxiv.org/abs/1909.01315}.

\bibitem[Wilson \& Martinez(2003)Wilson and Martinez]{wilson2003general}
D~Randall Wilson and Tony~R Martinez.
\newblock The general inefficiency of batch training for gradient descent learning.
\newblock \emph{Neural networks}, 16\penalty0 (10):\penalty0 1429--1451, 2003.

\bibitem[Wu et~al.(2020)Wu, Sun, Zhang, Xie, and Cui]{recommendersystemssurvey}
Shiwen Wu, Fei Sun, Wentao Zhang, Xu~Xie, and Bin Cui.
\newblock Graph neural networks in recommender systems: A survey, 2020.

\bibitem[Xu et~al.(2019)Xu, Wang, Chen, Tao, and Zhao]{Xu_2019}
Nuo Xu, Pinghui Wang, Long Chen, Jing Tao, and Junzhou Zhao.
\newblock {MR}-{GNN}: Multi-resolution and dual graph neural network for predicting structured entity interactions.
\newblock In \emph{Proceedings of the Twenty-Eighth International Joint Conference on Artificial Intelligence}, aug 2019.
\newblock \doi{10.24963/ijcai.2019/551}.

\bibitem[Ying et~al.(2018)Ying, He, Chen, Eksombatchai, Hamilton, and Leskovec]{pinsage}
Rex Ying, Ruining He, Kaifeng Chen, Pong Eksombatchai, William~L. Hamilton, and Jure Leskovec.
\newblock Graph convolutional neural networks for web-scale recommender systems.
\newblock In \emph{Proceedings of the 24th ACM SIGKDD International Conference on Knowledge Discovery and Data Mining}, KDD '18, pp.\  974–983, 2018.
\newblock \doi{10.1145/3219819.3219890}.

\bibitem[Zeng et~al.(2020)Zeng, Zhou, Srivastava, Kannan, and Prasanna]{graphsaint-iclr20}
Hanqing Zeng, Hongkuan Zhou, Ajitesh Srivastava, Rajgopal Kannan, and Viktor Prasanna.
\newblock {GraphSAINT}: Graph sampling based inductive learning method.
\newblock In \emph{International Conference on Learning Representations}, 2020.
\newblock URL \url{https://openreview.net/forum?id=BJe8pkHFwS}.

\bibitem[Zeng et~al.(2021)Zeng, Zhang, Xia, Srivastava, Malevich, Kannan, Prasanna, Jin, and Chen]{shaDow}
Hanqing Zeng, Muhan Zhang, Yinglong Xia, Ajitesh Srivastava, Andrey Malevich, Rajgopal Kannan, Viktor Prasanna, Long Jin, and Ren Chen.
\newblock Decoupling the depth and scope of graph neural networks.
\newblock In A.~Beygelzimer, Y.~Dauphin, P.~Liang, and J.~Wortman Vaughan (eds.), \emph{Advances in Neural Information Processing Systems}, 2021.
\newblock URL \url{https://openreview.net/forum?id=d0MtHWY0NZ}.

\bibitem[Zheng et~al.(2021)Zheng, Song, Yang, LaSalle, Su, Wang, Ma, and Karypis]{zheng2021distributed}
Da~Zheng, Xiang Song, Chengru Yang, Dominique LaSalle, Qidong Su, Minjie Wang, Chao Ma, and George Karypis.
\newblock Distributed hybrid cpu and gpu training for graph neural networks on billion-scale graphs.
\newblock \emph{arXiv preprint arXiv:2112.15345}, 2021.

\bibitem[Zhu et~al.(2019)Zhu, Zhao, Yang, Lin, Zhou, Ai, Li, and Zhou]{aligraph}
Rong Zhu, Kun Zhao, Hongxia Yang, Wei Lin, Chang Zhou, Baole Ai, Yong Li, and Jingren Zhou.
\newblock Aligraph: A comprehensive graph neural network platform.
\newblock \emph{Proc. VLDB Endow.}, 12\penalty0 (12):\penalty0 2094–2105, aug 2019.
\newblock \doi{10.14778/3352063.3352127}.

\end{thebibliography}
\bibliographystyle{tmlr}

\newpage
\appendix
\section{Appendix}

\subsection{Graph Sampling}
\label{subsecc:graph_sampling_background}

Below, we review three different sampling algorithms for minibatch training of
GNNs.
Our focus in this work is samplers whose expected number of sampled vertices is a
function of the batch size. All these methods
are applied recursively for GNN models with multiple layers.

\subsubsection{Neighbor Sampling (NS)}

Given a fanout parameter $k$ and a batch of seed vertices $S^0$,
NS by~\citep{neighborsampling} samples the neighborhoods of vertices
randomly.
Given a batch of vertices $S^0$, a vertex $s \in S^0$ with degree $d_s = |N(s)|$, if $d_s \leq k$, NS uses the full neighborhood $N(s)$, otherwise
it samples $k$ random neighbors for the vertex $s$.


\subsubsection{LABOR Sampling}

Given a fanout parameter $k$ and a batch of seed vertices $S^0$,
LABOR-0~\citep{Balin23-NeurIPS} samples the neighborhoods of vertices as
follows.
First, each vertex rolls a uniform random number $0 \leq r_{t} \leq 1$.
Given batch of vertices $S^0$, a vertex $s \in S^0$ with degree $d_s = |N(s)|$, the edge $(t \to s)$ is sampled
if $r_t \leq \frac{k}{d_s}$.
Since different seed vertices $\in S^0$ end up using the same random variate
$r_t$ for the same source vertex $t$, LABOR-0 samples fewer vertices than NS in expectation.

The LABOR-* algorithm is the importance sampling variant of LABOR-0 and samples an edge $(t \to s)$
if $r_t \leq c_s \pi_t$, where $\pi$ is importance sampling probabilities optimized to minimize
the expected number of sampled vertices and $c_s$ is a normalization factor. LABOR-*
samples fewer vertices than LABOR-0 in expectation.

Note that, choosing $k \geq \max_{s \in V} d_s$ corresponds to training with full
neighborhoods for both NS and LABOR methods.

\subsubsection{RandomWalk (RW) Sampling}

Given a walk length $o$, a restart probability $p$, number of random walks $a$, a fanout $k$, and
a batch of vertices $S^0$, a vertex $s \in S^0$, a {\em Random Walk}~\citep{pinsage} starts
from $s$ and each step picks a random neighbor $s'$ from $N(s)$. For the remaining
$o - 1$ steps, the next neighbor is picked from $N(s')$ with probability $1 - p$, otherwise
it is picked from $N(s)$. This process is repeated $a$ times for each seed vertex and lastly,
the top $k$ visited vertices become the {\em neighbors} of $s$ for the current layer.

Notice that random walks correspond to weighted neighbor sampling from a graph with adjacency matrix $\tilde{A} = \sum_{i = 1}^o A^i$,
where the weights of $\tilde{A}$ depend on the parameters $a$, $p$ and
$k$. Random walks give us the ability to sample from $\tilde{A}$ without actually
forming $\tilde{A}$.

\subsection{Work Monotonicity Theorem}
\label{subsecc:work_monotonicity}

\begin{theorem}
The work per epoch required to train a GNN model using minibatch training is
monotonically nonincreasing as the batch size increases.
\label{tha:work_monotonicity}
\end{theorem}

\begin{proof}
Given any $n \geq 2$, let's say we uniform randomly sample without replacement
$S^0 \subset V$, where $n=|S^0|$.
Now note that for any $S'^0 \subset S^0$, using the definition in (\ref{eq:s_l}), we have
$S'^l \subset S^l, \forall l$. We will take advantage of that and define
$S'^0 = S^0 \setminus \{s\}$ in following expression.

\begin{equation}
\begin{split}
    \sum_{\substack{s \in S^0 \\ S'^0 = S^0 \setminus \{s\}}} |S^l| - |S'^l| = \sum_{\substack{s \in S^0 \\ S'^0 = S^0 \setminus \{s\}}} \sum_{w \in S^l} \mathbbm{1}[w \notin S'^l] \\
    = \sum_{\substack{s \in S^0 \\ S'^0 = S^0 \setminus \{s\}}} \sum_{w \in \{s\}^l} \mathbbm{1}[w \notin S'^l] \\
    = \sum_{\substack{s \in S^0 \\ S'^0 = S^0 \setminus \{s\}}} \sum_{w \in \{s\}^l} \mathbbm{1}[w \notin \{s'\}^l, \forall s' \in S'^0] \\
    = \sum_{w \in S^l} \sum_{\substack{s \in S^0 \\ w \in \{s\}^l}} \mathbbm{1}[w \notin \{s'\}^l, \forall s' \in S^0 \setminus \{s\}]
    \label{eq:subset_argument}
\end{split}
\end{equation}

In the last expression, for a given $w \in S^l$, if there are two different elements $s, s' \in S^0$ such that
$w \in \{s\}^l$ and $w \in \{s'\}^l$, then the indicator expression will be $0$. It will be $1$ only if
$w \in \{s\}^l$ for a unique $s \in S^0$. So:

\begin{equation}
\begin{split}
    \sum_{w \in S^l} \sum_{\substack{s \in S^0 \\ w \in \{s\}^l}} \mathbbm{1}[w \notin \{s'\}^l, \forall s' \in S^0 \setminus \{s\}] = \sum_{\substack{w \in S^l \\ \exists! s \in S^0, w \in \{s\}^l}} 1 \\
    = |\{w \in S^l \mid w \in \{s\}^l, \exists! s \in S^0\}| \leq |S^l|
    \label{eq:subset_argument2}
\end{split}
\end{equation}

Using this, we can get:

\begin{gather*}
    \sum_{\substack{S'^0 \subset S^0 \\ |S'^0| + 1 = |S^0|}} |S^l| - |S'^l| \leq |S^l| \\
    \iff |S^0||S^l| - \sum_{\substack{S'^0 \subset S^0 \\ |S'^0| + 1 = |S^0|}} |S'^l| \leq |S^l| \\
    \iff |S^l|(|S^0| - 1) \leq \sum_{\substack{S'^0 \subset S^0 \\ |S'^0| + 1 = |S^0|}} |S'^l| \\
    \iff |S^l|(|S^0| - 1) \leq |S^0| E[|S'^l|] \\
    \iff \frac{|S^l|}{|S^0|} \leq \frac{E[S'^l]}{|S'^0|}
\end{gather*}

Since $S^0$ was uniformly randomly sampled
from $V$, its potential subsets $S'^0$ are also uniformly randomly picked from $V$ as a result. Then,
taking an expectation for the random sampling of $S^0 \subset V$, we conclude that
$\frac{E[|S^l|]}{|S^0|} \leq \frac{E[|S'^l|]}{|S'^0|}$, i.e., the expected work of batch size $n$ is
not greater than the work of batch of size $n-1$. This implies that the work with respect to batch size is a
monotonically nonincreasing function.
\end{proof}

\subsection{Overlap monotonicity}
\label{subsecc:overlap_monotonicity}

\begin{theorem}
The expected subgraph size $E[|S^l|]$ required to train a GNN model using minibatch training is a
concave function of batch size, $|S^0|$.
\label{tha:overlap_monotonicity}
\end{theorem}

\begin{proof}

Given any $n \geq 2$, let's say we uniformly randomly sample without replacement $S^0 \subset V$ of size $n$. Note that we use the $T^l$ and $T_2^l$ as defined in~\cref{eq:T_1_S,eq:T_2_S}.

\begin{equation}
\begin{split}
    |T^l(S^0)| - 2|T_2^l(S^0)| = \sum_{\substack{S'^0 \subset S^0 \\ |S'^0| + 1 = |S^0|}} |T^l(S^0)| - |T^l(S'^0)| \\
    = |S^0| |T^l(S^0)| - |S^0| E[|T^l(S'^0)|] \\
    \iff (|S^0| - 1) |T^l(S^0)| = |S^0| E[|T^l(S'^0)|] - 2|T_2^l(S^0)| \\
    \iff \frac{|T^l(S^0)|}{|S^0|} = \frac{E[|T^l(S'^0)|]}{|S'^0|} - \frac{2|T_2^l(S^0)|}{|S'^0||S^0|} \\
    \implies \frac{|T^l(S^0)|}{|S^0|} \leq \frac{E[|T^l(S'^0)|]}{|S'^0|}
\label{eq:overlap_monotonicity}
\end{split}
\end{equation}

where the first equality above is derived similar to~\cref{eq:subset_argument,eq:subset_argument2}. Overall, this means that the overlap between vertices increases as the batch size increases.
Utilizing our finding from~\cref{eq:subset_argument,eq:subset_argument2}, we have:

\begin{equation}
\begin{split}
    \sum_{\substack{S'^0 \subset S^0 \\ |S'^0| + 1 = |S^0|}} |S^l| - |S'^l| = |T^l(S^0)| \\
    \implies |S^0||S^l| - |S^0|E[|S'^l|] = |T^l(S^0)| \\
    \implies |S^l| = E[|S'^l|] + \frac{|T^l(S^0)|}{|S^0|} \\
    \implies E[|S^l|] = E[|S'^l|] + \frac{E[|T^l(S^0)|]}{|S^0|}
\end{split}
\end{equation}

Note that the last step involved taking expectations for the random sampling of $S^0$. See the recursion embedded in the equation above, the expected size of the subgraph $S^l$ with
batch size $|S^0|$ depends on the expected size of the subgraph $S'^l$ with batch size $|S^0|-1$. Expanding
the recursion, we get:

\begin{equation}
    E[|S^l|] = \sum_{i=1}^{|S^0|} \frac{E[|T^l(S^0_i)|]}{i}
\end{equation}

where $S^0_i$ is a random subset of $S^0$ of size $i$. Since $\frac{E[|T^l(S^0_i)|]}{i}$ is
monotonically nonincreasing as $i$ increases as we showed in~(\ref{eq:overlap_monotonicity}), we conclude that 
$E[|S^l|]$ is a concave function of the batch size, $|S^0|$.
\end{proof}

\begin{figure*}[ht]
    \centering
    \includegraphics[width=\linewidth]{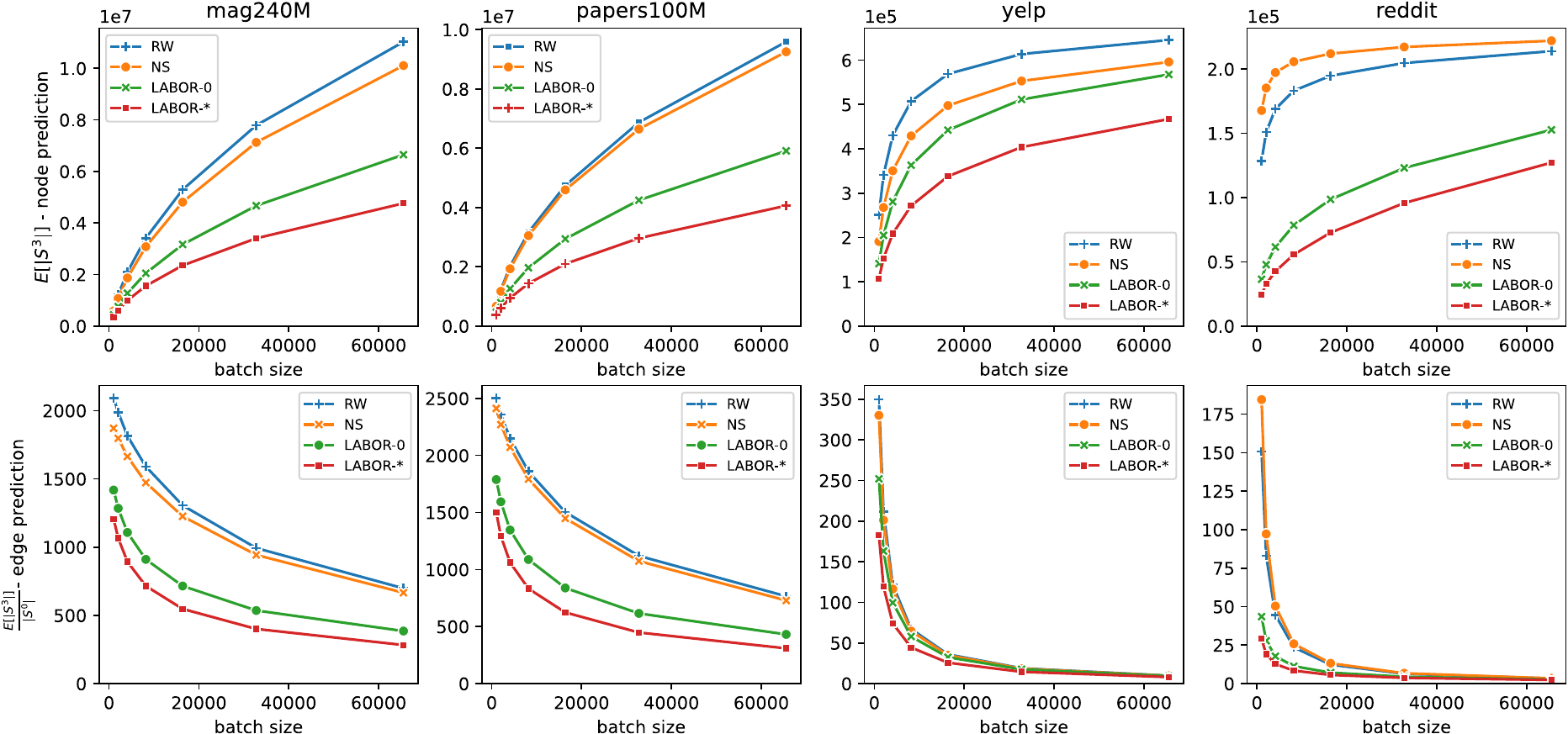}
    \caption{Monotonicity of the work. x axis shows the batch size, y axis shows
    $E[|S^3|]$ for node prediction (top row)
    and $\frac{E[|S^3|]}{|S^0|}$ for edge prediction (bottom row), where $E[|S^3|]$ denotes
    the expected number of vertices sampled in the 3rd layer and $|S^0|$ denotes
    the batch size.
    RW stands for Random Walks, NS stands for Neighbor Sampling, and LABOR-0/*
    stand for the two different variants of
    the LABOR sampling algorithm described in~\cref{subsecc:graph_sampling}. Completes~\cref{figc:num_input_nodes}.}
    \label{figc:num_input_nodes_2}
\end{figure*}

\subsection{Subgraph density monotonicity}
\label{subsecc:subgraph_density_monotonicity}

\begin{theorem}
The expected sampled subgraph density $\frac{E[|S_E|]}{|S|}$ is nondecreasing as the batch size $|S|$ increases.
\label{tha:subgraph_density_monotonicity}
\end{theorem}

\begin{proof}
Given $n \geq 1$ and any probability distribution $p$ on $V$ such that $p_s$
denotes the probability of including $s$ in our minibatch $S$ of size $|S|=n$, 
the probability of including the edge $t \to s$ becomes $p_t p_s$. Since
$p_s(n)$ is linear as a function of $n$, $\forall s \in V$, it can be decomposed as
$p_s(n) = p'_s n$.

\begin{equation}
    \frac{E[|S_E|]}{|S|} = \frac{1}{|S|}\sum_{(t \to s) \in E} \mathbb{P}(t \in S \land s \in S) = \frac{1}{|S|}\sum_{(t \to s) \in E} p_t p_s = \frac{1}{n}\sum_{(t \to s) \in E} p'_t n p'_s n = n \sum_{(t \to s) \in E} p'_t p'_s
\end{equation}

We can see that the expected sampled subgraph density is a linear function of the 
batch size $|S|=n$ and this implies that it is nondecreasing as the batch size $|S|$ 
increases.
\end{proof}

\subsection{Experimental Setup}
\label{subsecc:experimental_setup}

\textbf{Setup:} In our experiments, we use the following datasets:
reddit~\citep{neighborsampling}, papers100M~\citep{ogb2020}, mag240M~\citep{ogblsc}, yelp and
flickr~\citep{graphsaint-iclr20}, and their details are given in~\cref{tabc:dataset}. We use
Neighbor Sampling (NS)~\citep{neighborsampling}, LABOR
Sampling~\citep{Balin23-NeurIPS} and Random Walks (RW)~\citep{pinsage} to form minibatches.
We used a fanout of $k=10$ for the samplers. In addition, Random Walks
used length of $o=3$, restart probability $p = 0.5$ and
number of random walks from each seed $a=100$.
All our experiments involve a GCN model with $L=3$
layers~\citep{neighborsampling}, with 1024 hidden dimension for mag240M and papers100M
and 256 for the rest. Additionally, papers100M and mag240M datasets were made undirected graphs
for all experiments and this is reflected in the reported edge counts in~\cref{tabc:dataset}. Input
features of mag240M are stored with the 16-bit floating point type.
We use the Adam optimizer~\citep{Kingma2014} with $10^{-3}$ learning rate in all the experiments.

\textbf{Implementation:}
We implemented
\footnote{Source code is available at \url{https://github.com/GT-TDAlab/dgl-coop/tree/dist_graph_squashed_wip_cache}}
our experimental code using C++ and Python in the DGL framework~\citep{wang2019deep}
with the Pytorch backend~\citep{Paszke2019PyTorchAI}.
All our experiments were repeated 50 times
and averages are presented. Early stopping
was used during model training runs.
So as we go to the right along the
x-axis, the variance of our convergence plots increases because the number of runs that were ongoing is decreasing.

\subsection{Complexity Analysis (cont.)}
\label{subsecc:complexity_analysis}

\begin{figure}[ht]
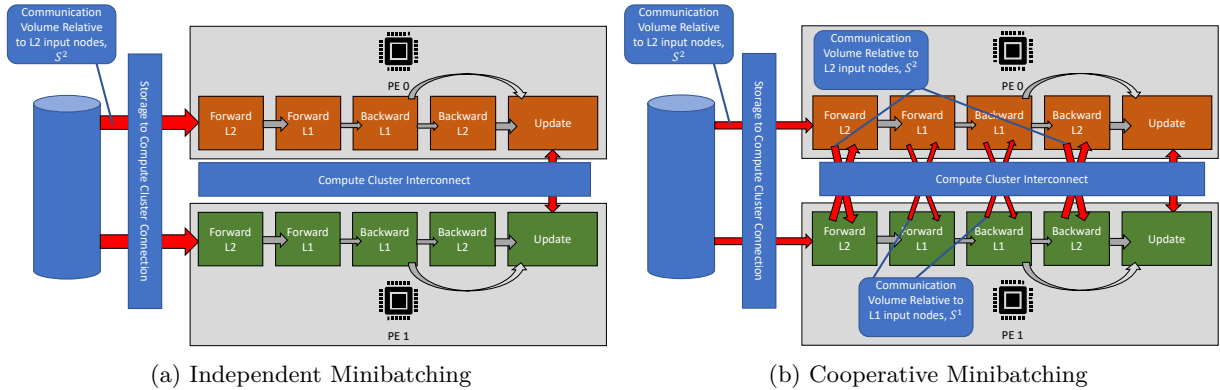

    \centering
\begin{subfigure}{.49\linewidth}
    \includegraphics[width=\linewidth,page=4,trim=0.9cm 0.4cm 2.0cm 1.2cm]{figures/cooperative_minibatching_figures.pdf}
    \caption{Independent Minibatching}
    \label{figc:independent_minibatching}
\end{subfigure}
\begin{subfigure}{.49\linewidth}
    \centering
    \includegraphics[width=\linewidth,page=5,trim=0.9cm 0.4cm 2.0cm 1.2cm]{figures/cooperative_minibatching_figures.pdf}
    \caption{Cooperative Minibatching}
    \label{figc:cooperative_minibatching}
\end{subfigure}
\caption{A comparison of Independent and Cooperative Minibatching approaches in the feature loading and forward-backward GNN stages.
The thickness of the red arrows indicates the data volume. Due to redundant
vertices across PEs, Independent Minibatching wastes (PCI-e) bandwidth for vertex embedding copies from Storage. Moreover, the PEs perform identical computations for
redundant edges across PEs in Independent Minibatching.}
\end{figure}

Our goal in this section is to empirically show the work reduction enjoyed by cooperative minibatching over independent
minibatching by reporting the number of vertices and edges processed per PE. We also report the number of vertices
that are communicated for cooperative minibatching during its all-to-all calls in~\cref{alg:cooperative_minibatching,figc:cooperative_minibatching}.
The results are given in~\cref{tabc:num_sampled}.

\begin{table*}[ht]
\caption{Average number of vertices and edges sampled in different layers with LABOR-0 per PE, reduced by taking the
maximum over 4 PEs (All
the numbers are in thousands, lower is better) with batch size $|S^0| = 1024$. $c|\tilde{S}^l|$ shows the number of vertices communicated at layer $l$.
Papers and mag were used as short versions of papers100M/GCN and mag240M/R-GCN dataset model pairs respectively.
Last column shows forward-backward (F/B) runtime in ms.}
\label{tabc:num_sampled}
\begin{adjustbox}{width=\linewidth,center}
\begin{tabular}{c | c | c | c | c | c | c | c | c | c | c | c | c}
\toprule
\textbf{Dataset} & \textbf{Part.} & \textbf{I/C} & $\bm{|S^3|}$ & $\bm{c|\tilde{S}^3|}$ & $\bm{|\tilde{S}^3|}$ & $\bm{|E^2|}$ & $\bm{|S^2|}$ & $\bm{c|\tilde{S}^2|}$ & $\bm{|\tilde{S}^2|}$ & $\bm{|E^1|}$ & $\bm{|S^1|}$ & \textbf{F/B} \\
\midrule
 \multirow{3}{*}{papers} & random & Indep & 463 & 0 & 463 & 730 & 74.8 & 0 & 74.8 & 93.6 & 9.63 & 8.9 \\
                             & random & Coop & 318 & 311 & 463 & 608 & 62.4 & 56.8 & 82.8 & 89.9 & 9.28 & 13.0 \\
                             & metis & Coop & 328 & 179 & 402 & 615 & 63.1 & 34.0 & 73.8 & 90.8 & 9.35 & 12.0 \\
\midrule
 \multirow{3}{*}{mag} & random & Indep & 443 & 0 & 443 & 647 & 67.9 & 0 & 67.9 & 82.0 & 8.78 & 199.9 \\
                          & random & Coop& 324 & 310 & 459 & 566 & 59.8 & 53.1 & 77.3 & 80.4 & 8.62 & 183.3 \\
                          & metis & Coop & 334 & 178 & 419 & 576 & 60.6 & 31.0 & 71.3 & 81.8 & 8.80 & 185.1 \\
\bottomrule
\end{tabular}
\end{adjustbox}
\end{table*}

Looking at~\cref{tabc:complexity,tabc:num_sampled}, we make the following observations:
\begin{enumerate}

\item All runtime complexities for cooperative minibatching scales with $|S^l_p(B)| = |S^l(B)|\frac{1}{P}$ and $|E^l_p(B)| = |E^l(B)|\frac{1}{P} \leq |S^l(B)|\frac{k}{P}$ and for independent minibatching with $|S^l(\frac{B}{P})|$ and $|E^l(\frac{B}{P})| \leq |S^l(\frac{B}{P})|k$, for a sampler with fanout $k$. Since $E[|S^l(B)|]$ is a
concave function, $E[|S^l(B)|]\frac{1}{P} \leq E[|S^l(\frac{B}{P})|]$, and this corresponds to looking 
at~\cref{figc:num_input_nodes}
first row with $x=B$ for coop and $x=\frac{B}{P}$ for independent if one wanted to guess how 
their runtime would change with changing $B$ and $P$. For an example, all the runtime numbers we 
have provided in the~\cref{tabc:cooperative_runtimes} are for 4 GPUs. Going from 4 to 8 would increase the edge
of cooperative over independent even more, see~\cref{tabc:cooperative_speedup}.

\item Sampling and Feature loading over PCI-e requires $\alpha \gg \beta$ for cooperative to get a speedup over independent.

\item In order for cooperative F/B to improve against independent, we need that
$\frac{\alpha}{c} > \frac{\gamma}{M}$.

\item Cross edge ratio $c$ reduces all communication between PEs. In particular, graph partitioning will lower both $c$ 
and $|\tilde{S}^l_p(B)|$, lowering the communication overhead, see $\bm{c|\tilde{S}^l|}$ 
columns in~\cref{tabc:num_sampled}.

\item The model complexity $M$ is small for the GCN model (papers100M) but large for the R-GCN 
model (mag240M), as shown by the F/B runtime numbers in Table 2. Also, the communication overhead
between the two scenarios is similar, meaning communication can take from upto 30\% to less than a few percent 
depending on $M$. For the papers100M dataset, communication makes up more of the runtime, so graph partitioning
helps bring the F/B runtime down. However, the load imbalance caused by graph partitioning
slows down the F/B runtime despite lowered communication costs for the mag240M dataset.

\end{enumerate}

\subsection{Smoothed Dependent Minibatching}
\label{subsecc:smoothed_dependent_minibatching}

As described in \cref{secc:background}, NS algorithm works by using the random variate
$r_{ts}$ for each edge $(t \to s)$.
Being part of the same minibatch means that a single random variate $r_{ts}$ will be used for each edge.
To generate these random variates, we initialize a Pseudo Random Number Generator (PRNG) with a 
random seed $z$ along with $t$ and $s$ to ensure that the first 
rolled random number $r_{ts}$ from the PRNG stays fixed when the random seed $z$ is fixed.
Given random seeds $z_1$ and $z_2$, let's say
we wanted to use $z_1$ for the first $\kappa$ iterations and would later switch to the seed $z_2$.
This switch can be made smoothly by interpolating between the random seeds $z_1$ and $z_2$ while ensuring that
the resulting sampled random numbers are uniformly distributed. If we are processing the
batch number $i < \kappa$ in the group of $\kappa$ batches, then we want the contribution of $z_2$ to be $c = \frac{i}{\kappa}$, while the
contribution of $z_1$ is $1 - c$. We can sample $n_{ts}^1 \sim \mathcal{N}(0, 1)$ with seed $z_1$ and
$n_{ts}^2 \sim \mathcal{N}(0, 1)$ with seed $z_2$. Then we combine them as
\[
n_{ts}(c) = \cos(\frac{c\pi}{2}) n_{ts}^1 + \sin(\frac{c\pi}{2}) n_{ts}^2
\]
\noindent note that $n_{ts}(0) = n_{ts}^1$, $n_{ts}(1) = n_{ts}^2$ and $n_{ts}(c) \sim \mathcal{N}(0, 1), \forall c \in \mathbb{R}$ also, then we can set
$r_{ts} = \Phi(n_{ts}(c))$, where $\Phi(x) = \mathbb{P}(Z \leq x)$ for $Z \sim \mathcal{N}(0, 1)$, to get $r_{ts} \sim U(0, 1)$ that the NS algorithm can use. Dropping $s$
from all the notation above gives the version for LABOR.
In this way, the random variates change slowly as we are switching from one
group of $\kappa$ batches to another. When $i = \kappa$, we let $z_1 \leftarrow z_2$ and
initialize $z_2$ with another random seed.
To use this approach, only the random variate generation logic
needs modification, making its implementation for any sampling algorithm straightforward compared to
the nested approach initially described. \cref{figc:dependent_minibatching_example} shows an example of this approach for $\kappa=2$.

\subsection{Dependent batches (cont.)}
\label{subsecc:dependent_batches_cont}

Looking at the training loss and validation F1-score with
sampling curves in~\cref{figc:same_loss_cache_sampling}, we notice that the performance gets better as $\kappa$ is increased. This
is due to the fact that a vertex's neighborhood changes slower and slower as $\kappa$ is increased,
the limiting case being $\kappa=\infty$, in which case the sampled neighborhoods are unchanging. This
makes training easier so $\kappa=\infty$ case leads the pack in the training loss and validation F1-
score with sampling curves.

\begin{figure*}[ht]
    \centering
    \includegraphics[width=\linewidth]{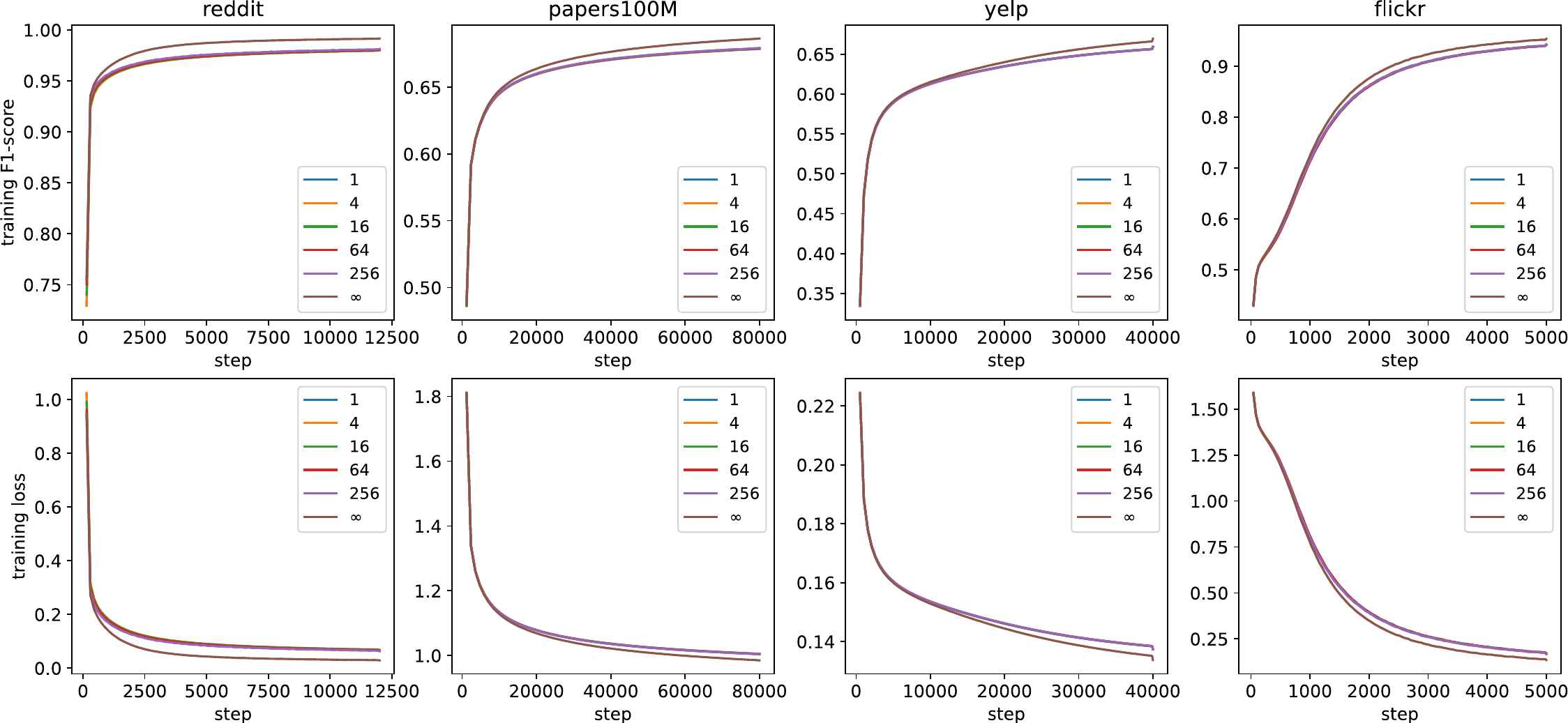}
    \caption{LABOR-0 sampling algorithm with $1024$ batch size and varying $\kappa$ dependent minibatches,
    $\kappa=\infty$ denotes infinite dependency, meaning the neighborhood
    sampled for a vertex stays static during training. The first row shows the training F1-score
    with the dependent sampler. The second row shows the training loss curve. Completes~\cref{figc:same_loss_cache}.}
    \label{figc:same_loss_cache_sampling}
\end{figure*}

\subsection[Comparing a single batch vs P independent batches convergence]{Comparing a single batch vs $P$ independent batches convergence}
\label{subsecc:coop_indep_conv}

We investigate whether training with a single large batch in $P$-GPU training shows any convergence differences
to the current approach of using $P$ separate batches for each of the GPUs. We use a global batch size of $4096$ and divide
a batch into $P \leq 8$ independent batches, with each batch having a size of $\frac{4096}{P}$.
We use NS and LABOR-0 samplers with fanouts of $k = 10$ for each of the $3$ layers. \cref{figc:coop_vs_indep} shows that there are no significant differences
between the two approaches, we present the results averaged over the samplers to save space.

\begin{figure*}[ht]
    \centering
    \includegraphics[width=\linewidth]{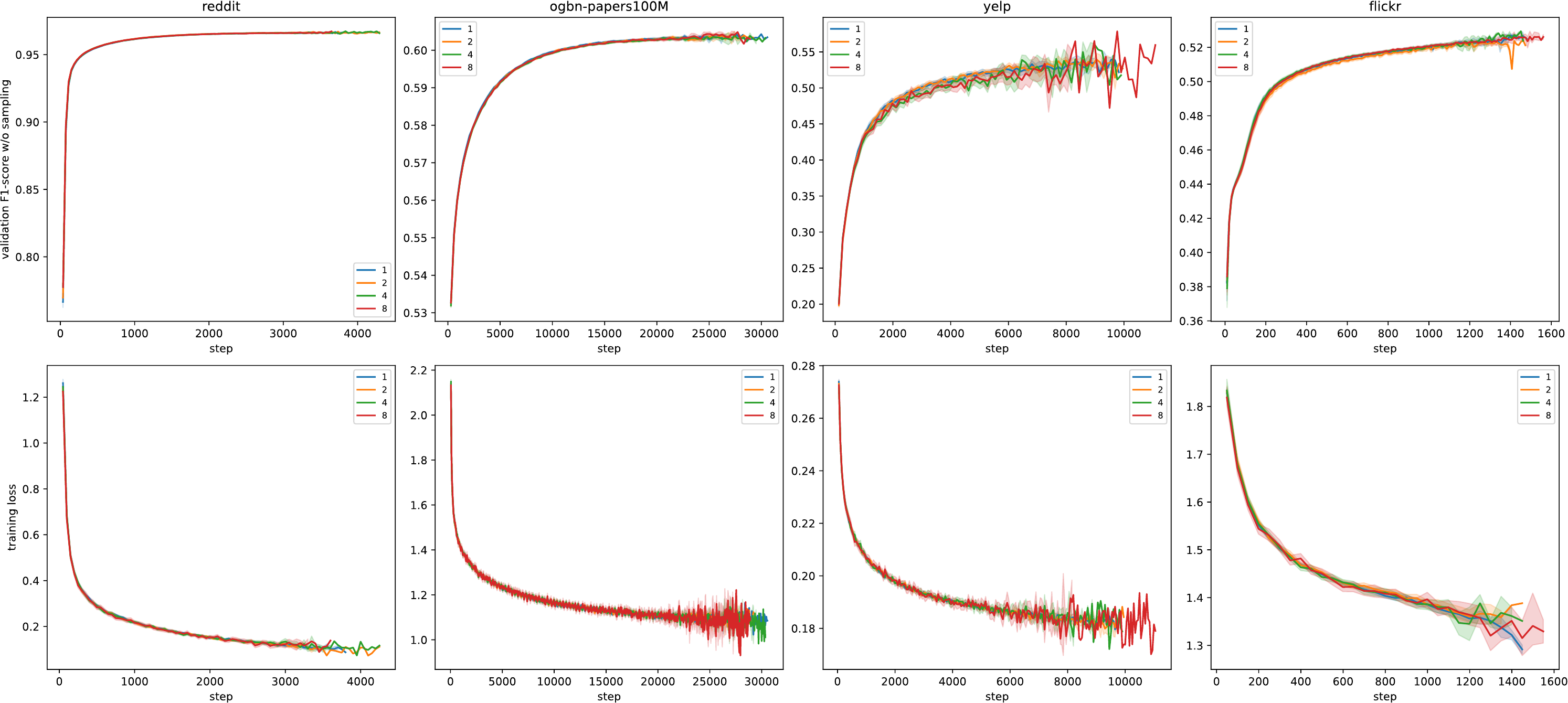}
    \caption{Convergence difference between cooperative vs independent minibatching with a global batch size of 4096 averaged over Neighbor and LABOR-0 samplers.}
    \label{figc:coop_vs_indep}
\end{figure*}

\subsection{Redundancy Free GCN aggregation}
\label{subsecc:redundancy_free_connection}

\cite{jia2020} proposes a method of reducing processing during neighborhood aggregation by finding common sub-neighborhoods
among aggregated vertices, whether using full-batch or minibatch training. That is, if two vertices have the neighborhoods
$\{A, B, C\}$ and $\{B, C, D\}$, and a summation operator is used for aggregation, then instead of computing four additions:
$A + B + C$ and $B + C + D$ concurrently, the three additions $BC = B + C$, $A + BC$, and $BC + D$ can be computed. This approach is orthogonal to the approaches proposed in~\cref{secc:cooperative_minibatching} in that it reduces redundant aggregation steps, whereas our approach reduces redundant input nodes and edges in parallel computations. As such, the two
approaches could be employed together.




\subsection{Limitations}
\label{subsecc:limitations}

Our proposed Cooperative Minibatching approach requires a relatively fast interconnect between the Processing Elements.
Modern multi-GPU systems usually have such interconnect between the GPUs.
Nowadays, such fast interconnects are being extended to the multi-node setting, for up to 72 GPUs connected via NVLink~\citep{nvl72}.

On distributed training, each computing node equipped with multiple cores and/or GPUs, the interconnect between cores and/or GPUs of the same node, are also relatively much faster than the interconnect among nodes.

The proposed dependent minibatching approach does not have any such limitation.

\end{document}